\documentclass[letterpaper]{article} % DO NOT CHANGE THIS
\usepackage{aaai25}  % DO NOT CHANGE THIS
\usepackage{times}   % DO NOT CHANGE THIS
\usepackage{helvet}  % DO NOT CHANGE THIS
\usepackage{courier}  % DO NOT CHANGE THIS
\usepackage[hyphens]{url}  % DO NOT CHANGE THIS
\usepackage{graphicx} % DO NOT CHANGE THIS
\usepackage{natbib}  % DO NOT CHANGE THIS AND DO NOT ADD ANY OPTIONS TO IT
\usepackage{caption} % DO NOT CHANGE THIS AND DO NOT ADD ANY OPTIONS TO IT
\usepackage{algorithm}
\usepackage{algorithmic}
\usepackage{dsfont}

\usepackage{newfloat}
\usepackage{listings}
\DeclareCaptionStyle{ruled}{labelfont=normalfont,labelsep=colon,strut=off} % DO NOT CHANGE THIS
\floatstyle{ruled}
\newfloat{listing}{tb}{lst}{}
\floatname{listing}{Listing}
\PassOptionsToPackage{numbers, compress}{natbib}
\usepackage[utf8]{inputenc} % allow utf-8 input
\usepackage{booktabs}       % professional-quality tables
\usepackage{amsfonts}       % blackboard math symbols
\usepackage{nicefrac}       % compact symbols for 1/2, etc.
\usepackage{microtype}      % microtypography
\usepackage{xcolor}         % colors
\usepackage{mathtools}
\usepackage{amsmath}
\usepackage{amssymb}
\usepackage{amsthm}
\usepackage[font={small}]{caption}
\usepackage[font={scriptsize}]{subcaption}
\usepackage{import}
\usepackage{xifthen}
\usepackage{pdfpages}
\usepackage{adjustbox}

\usepackage{array} % For centering text in cells
\usepackage{siunitx} % For aligning numbers by decimal points
\usepackage{multirow} % For multirow cells
\usepackage{tabularx}

\usepackage{eqparbox}

\usepackage{paralist}
\usepackage{enumitem}

\graphicspath{ {./res} }

\newcommand{\learner}{\mathcal{L}}
\newcommand{\mO}{\mathcal{O}}
\newcommand{\mS}{{\mO_{low}}}
\newcommand{\mF}{{\mO_{high}}}
\newcommand{\cS}{{S^{}_\learner}}

\newcommand{\sL}{{\mathbb{L}}}
\newcommand{\sE}{{\mathbb{E}}}
\newcommand{\sM}{{\mathbb{M}}}

\newcommand{\sU}{{\mathbb{U}}}

\newcommand{\sX}{{\mathbb{X}}}
\newcommand{\sY}{{\mathbb{Y}}}

\newcommand{\tQ}{{\textit{Q}}}

\newcommand\norm[1]{\left\lVert#1\right\rVert}

\newcommand{\charac}[1]{{\mathds{1}_{#1}}}%{\mathbbm{1}

\newcommand{\probcover}{\emph{ProbCover}}

\newcommand{\MethodName}{\emph{DCoM}}
\newcommand{\app}{{Appendix~}}

\definecolor{orange}{rgb}{1, 0.75, 0.35}
\definecolor{blue}{rgb}{0.36, 0.54, 1}
\definecolor{red}{rgb}{1, 0.54, 0.36}
\definecolor{green}{rgb}{0.2, 0.7, 0.7}

\definecolor{pink}{rgb}{1.0, 0.4, 0.4}

\newcommand{\myparagraph}[1]{\smallskip\noindent\textbf{#1}}

\theoremstyle{definition}
\newtheorem{definition}{Definition}[section]

\newtheorem*{lemma*}{Lemma}

\newtheorem*{proposition}{Proposition}

\newtheorem*{theorem*}{Theorem}

\newtheorem*{corollary*}{Corollary}
\title{\MethodName: Competence-Driven Adaptive Active Learning}
\author {
    Inbal Mishal\textsuperscript{\ensuremath\dagger}, 
    Daphna Weinshall\textsuperscript{\ensuremath\dagger}
}
\affiliations {
   School of Computer Science \& Engineering\textsuperscript{\ensuremath\dagger} \\ 
   The Hebrew University of Jerusalem \\  Jerusalem 91904, Israel\\
   {\tt\small \{inbal.mishal,daphna\}@mail.huji.ac.il}\\
}
\begin{document}

\maketitle

\begin{abstract}
Deep Active Learning techniques can be effective in reducing annotation costs for training deep models. However, their effectiveness in low- and high-budget scenarios seems to require different strategies, and achieving optimal results across varying budget scenarios remains a challenge.   In this study, we introduce Dynamic Coverage \& Margin mix ($\MethodName$), a novel active learning approach designed to bridge this gap. Unlike existing strategies, \MethodName\ dynamically adjusts its strategy, while taking into account the competence of the current model. Through theoretical analysis and empirical evaluations on diverse datasets, including challenging computer vision tasks, we demonstrate $\MethodName$'s ability to overcome the cold start problem and consistently improve results across different budgetary constraints. Thus \MethodName\ achieves state-of-the-art performance in both low- and high-budget regimes. %Code is available at \href{https://github.com/avihu111/TypiClust}{https://github.com/avihu111/TypiClust}.
\end{abstract}

\section{Introduction}
\label{sec:intro}
Deep Learning (DL) algorithms require large amounts of data to achieve optimal results. In some situations, there is an abundance of unlabeled data but limited capacity to label it. In fields such as medical imaging, an invaluable resource — doctors themselves — serves as a costly oracle. 

Active learning (AL) algorithms aim to address this challenge by reducing the labeling burden and enhancing its effectiveness (see illustration in Fig.~\ref{fig:AL_illustration}). Unlike traditional supervised learning frameworks, AL can influence the construction of the labeled set, potentially by leveraging knowledge about the current learner. Accordingly, the initial goal of AL is to select $q$ examples to be annotated, where $q$ represents the number of examples that can be sent to the oracle. AL has already demonstrated tangible contributions across various domains such as computer vision tasks \citep[e.g.][]{yuan2023active}, NLP \citep[e.g.][]{siddhant2018deep, kaseb2023active} and medical imaging \citep[e.g.][]{s22031184}. These examples highlight the importance of advancing AL to achieve even greater impact.

\begin{figure}
    \centering
    \includegraphics[width=\columnwidth]{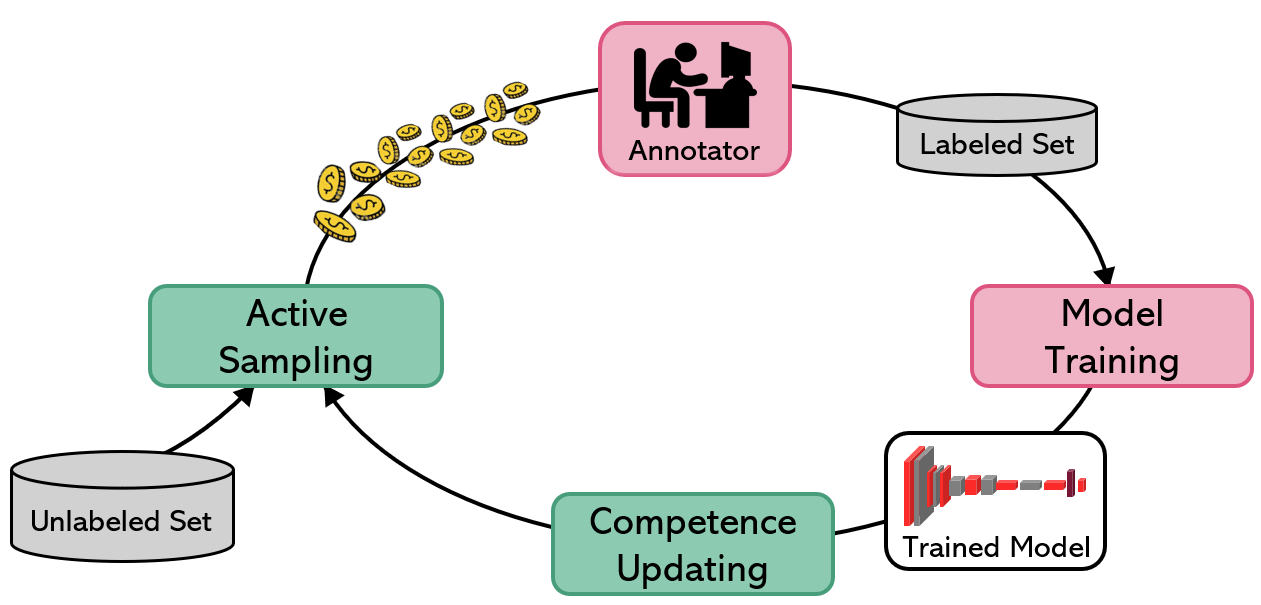}
    \caption{Illustration of the AL iterative process: \MethodName\ introduces an additional step to estimate the current learner's competence, which is then utilized during the active sampling phase.}
    \label{fig:AL_illustration}
    \vspace{-.3cm}
\end{figure}

The choice of an active learning (AL) strategy depends on both the learner's inductive biases and the nature of the problem. Recent research indicates that the optimal AL strategy varies with the budget size, where the budget refers to the size of the training set. For larger budgets, uncertainty and diversity sampling methods are most effective, while for smaller budgets, strategies focused on typicality and diversity are more suitable. However, no single AL strategy is universally optimal across all budget levels. Our study addresses this challenge by dynamically selecting the best examples based on the learner's current state and the available budget. 

Recent research, as briefly reviewed below, often categorizes active sampling methods based on budget regimes. In contrast, our study takes a different approach by emphasizing the need to adapt sample selection methods according to the learner's competence. We aim to develop a single algorithm that dynamically adjusts its sampling strategy in response to the learner's evolving abilities. This shift from a budget-focused to a competence-driven approach represents a significant advancement in active learning research, allowing for a more adaptive strategy that optimizes performance based on measurable learner's capabilities.

%In this work, we introduce a novel method called Dynamic Coverage \& Margin mix (\MethodName), which dynamically adjusts its selection strategy to deliver optimal results. We begin by examining a simplified theoretical framework (Section~\ref{sec:theory}), where we integrate a typicality test with an uncertainty test based on the budget and its coverage. Inspired by the insights gained from this analysis, we introduce \MethodName\ (Section~\ref{sec:method}), which combines these characteristics. \MethodName\ aims to offer a flexible solution for any budget by selecting the most appropriate examples that are most suitable to the current budget. We validate \MethodName\ through an extensive empirical study using various vision datasets (Section~\ref{sec:results}). Our findings indicate that \MethodName\ effectively achieves superior performance across all budget ranges.

In this work, we introduce a novel method called Dynamic Coverage \& Margin Mix (\MethodName), which dynamically adjusts its selection strategy by minimizing a new loss function that incorporates a measure of competence. We start by exploring a simplified theoretical framework (Section~\ref{sec:theory}), where we integrate a typicality test with an uncertainty test based on the budget and a notion of coverage. Building on the insights from this analysis, we introduce \MethodName\ (Section~\ref{sec:method}), which combines these elements to provide a flexible solution for any budget. Thus, \MethodName\ selects the most appropriate examples suited to the current budget. We validate \MethodName\ through an extensive empirical study using various vision datasets (Section~\ref{sec:results}), demonstrating that \MethodName\ consistently achieves superior performance across all budget ranges.

\myparagraph{Relation to prior art.}
Over the past years, active learning has been an active area of research \citep{wang2023comprehensive, settles2009active, cohn1996active, tharwat2023survey, xie2023active, schmidt2024unifiedapproachactivelearning}.
The realm of Active Learning (AL) is segmented into three problem settings: membership query synthesis, stream-based selective sampling, and pool-based active learning \citep{settles2009active}. Much of the recent AL work follows the pool-based setting, where samples are selected from a large pool of unlabeled data and annotations are obtained from an oracle. In this setting, AL algorithms operate within a given budget ($\mathit{b}$), selecting a subset of unlabeled examples to send to an oracle for labeling. This selection process aims to identify examples that will maximize the model's performance. This process may be repeated iteratively, gradually increasing budget $b$.

In discussions of pool-based AL, two key characteristics that play a crucial role during the active sampling process are \emph{uncertainty} and \emph{diversity}. Some AL algorithms focus solely on uncertainty sampling \citep{lewis1994sequential, shannon1948A, cho2023querying, zhao2021uncertaintyaware, houlsby2011bayesian, gissin2019discriminative, 8297020, woo2022active, gal2017deep, parvaneh2022active, jung2023simple}, while others prioritize diversity sampling \citep{sener2018active, hu2010off, yehuda2022active, hacohen2022active}. Other algorithms combine both uncertainty and diversity in their approaches \citep{yang2015multi, chen2024making, wen2023ntkcpl, kirsch2019batchbald, Chowdhury_2024_WACV, ash2019deep, kim2021task}.

Over time, the cold start problem has emerged, highlighting the difficulties faced by AL algorithms when operating with a small budget. It seems that  \emph{uncertainty} and \emph{diversity} are ill-suited to address this challenge. Accordingly, \citet{chen2024making} describe the cold start problem as a combination of two issues: unbalanced sampling across different classes (which relates to diversity). and unbalanced sampling within each class, where uncertain examples are favored over certain ones at the outset.

In low-budget scenarios, typical examples often offer the highest benefits for training the model \citep{hacohen2022active, yehuda2022active, chen2024making}. This implies that in such contexts, AL algorithms should aim to maximize both certainty and diversity. Conversely, in high-budget scenarios, the emphasis shifts towards maximizing uncertainty while maintaining diversity. While \citep{hacohen2024select} proposes a method to select the most suitable approach for the current state, the method involves the comparison of multiple active learning strategies at each state, rather than recommending a single universal approach.

\myparagraph{Summary of contribution.}  
\begin{inparaenum}[(i)] 
    \item Expand a theoretical framework to analyze Active Learning (AL) strategies in embedding spaces.
    \item Introduce \MethodName, a universal strategy that significantly outperforms previous methods in low and medium budget scenarios, and matches them in high budget scenarios.
    \item Initiate a transition from a budget-based discussion to the learner's competence, while introducing coverage as a measure designed to predict competence.
\end{inparaenum}

\section{Theoretical Framework}
\label{sec:theory}

%The framework of active learning adopted here can be formalized as follows: 
Let $\sX$ denote the input domain, $\sY$ the target domain, and $\learner: \sX\to\sY$ the trained learner. Before seeing any labels, $\learner$ may be a random feasible hypothesis, or it may be initialized by either transfer learning or self training. Let $b$ denote the size of the labeled training set, aka \emph{budget}.

Active learning may involve a single step, or an iterative process with repeated active learning steps. In each AL step, the learner actively seeks labels by choosing a set of unlabeled points as a query set, to be labeled by an oracle/teacher. Subsequently, this set of labeled points is added to the learner's supervised training set, and the learner is either retrained or fine-tuned. 

As discussed in Section~\ref{sec:intro}, it has been shown that different active learning strategies are suited for different budgets $b$. %Henceforth, $b$ refers to the number of labeled examples known to the learner at the beginning of an active learning step. 
With a low budget $b$, strategies that do not rely on the outcome of the learner are most suitable. When $b$ is high, it is beneficial to consider the confidence of the learner when choosing an effective query set. What makes a certain budget $b$ high or low is left vague in previous work, as it depends on both the hypothesis space (or architecture) and the dataset. 

To achieve an effective active learning protocol suitable for all learners, irrespective of budget, we aim to devise a universal objective function, whose minimization is used to select the query set. %Next, we discuss its envisioned form for selecting the query set at the beginning of each AL step.
We begin by proposing to replace the notion of \emph{budget} with the notion of \emph{competence}, designed to track the generalization ability of $\learner$ when trained with budget $b$. The notion of competence, we hypothesize, is more universal and less dependent on the type of learner and specific dataset (see discussion in Section~\ref{sec:AL-mixed}). 

We then rephrase the intuition stated above (see Section~\ref{sec:intro}) as follows: When the learner's competence is low, the objective function should prioritize typicality and diversity of selected queries irrespective of the learner's predictions. When the learner's competence is high, its uncertainty in prediction should be given high priority.

More formally, let ${\mO(x)}$ denote the desired objective function for query selection, where $x \in \sX$ is unlabeled. Let ${\mS(x)}$ and ${\mF(x)}$ denote the objective functions suitable for a learner with low competence and high competence respectively. Let $\cS$ denote a score, which captures the competence of learner $\learner$. The proposed objective function can now be written as follows $\forall x \in \sX$:
\begin{equation}
\label{eq:obj_func_def}
{\mO}(x) = (1-\cS)\cdot{{\mS}(x)} + \cS\cdot{{\mF}(x)}
\end{equation}

After the introduction of necessary preliminaries in Section~\ref{sec:prelim} we discuss in Section~\ref{sec:new_bound_generalization_error} the design of $\mS(x)$. We begin with the point coverage framework introduced by \citet{yehuda2022active}. This approach relies on self-supervised data representation, which is blind to the learner's performance. We refine their analysis and obtain an improved generalization bound for the Nearest Neighbor classification model that depends on the local geometry of the data.

Next, in Section~\ref{sec:AL-mixed} we discuss our choice of $\cS$, and connect it to the notion of budget discussed in previous work. Finally, we select $\mF(x)$ to equal one minus the normalized lowest response between the two highest softmax outputs of $\learner$ at $x$, or \emph{Margin}, as this is a common uncertainty measure.

\subsection{Preliminaries}
\label{sec:prelim}

\myparagraph{Notations.}
Let $P$ denote the underlying probability distribution of data $\sX$. Assume that a true labeling function $f:\sX \to \sY$ exists. Let $\sU\subseteq \sX$ denote the unlabeled set of points, and $\sL \subseteq \sX$ the labeled set, such that $\sX=\sU \cup \sL$. Here $\left|\sL\right|=b \le m$ is the annotation budget where $\left|\sX\right|=m$.

Let $B_{\delta}\left(x\right)=\left\{ x'\in\sX:\norm{x'-x}_{2}\le\delta\right\}$ denote a ball of radius $\delta$ centered at $x$. Let $C\equiv C(\sL,\delta)=\bigcup_{x\in \sL} B_\delta(x)$ denote the region covered by $\delta$-balls centered at the labeled examples in $\sL$. We call $C(\sL,\delta)$ the \emph{covered region}, where $P\left(C\right)$ denotes its probability. Let $f^{N}$ denote the 1-NN classifier, and $\learner$ denote our current learner -- a 1-NN classifier based on $\sL$.
\medskip

\begin{definition}
We say that a ball $B_\delta(x)$ is \emph{pure} if $\forall x'\in B_\delta(x): f(x')=f(x)$. 
\end{definition}

\begin{definition}
We define the \emph{purity} of $\delta$ as 
$$\pi\left(\delta\right)=P\left(\left\{ x \in \sX:B_{\delta}\left(x\right)\text{ is pure}\right\} \right).$$ 
Note that $\pi(\delta)$ is monotonically decreasing, as can be readily verified.
\end{definition}

In \citep{yehuda2022active} it is shown that the generalization error of the 1-NN classifier $f^{N}$ is bounded as follows
\begin{equation}
\label{eq:old_bound}
 \sE \left[f^{N}(x)\neq f(x)\right]\leq \left(1-P[C(\sL,\delta)]\right)+(1-\pi(\delta)).
\end{equation}
%\citep{yehuda2022active} theoretical analysis focuses on the low budget regime, so that it claims that $\pi(\delta)$ is negligible. This claim is not entirely right for the low budget analysis, but it even more incorrect when creating a multi-budget analysis. In the following analysis, we approximate a lower bound for $\pi(\delta)$. The new bound creates new and more accurate bound of the 1-NN generalization error.
Subsequently, an algorithm is proposed that minimizes the first term in (\ref{eq:old_bound}) by maximizing the coverage probability $P[C(\sL,\delta)]$, while ignoring the second term that is assumed to be fixed.

Below, we begin by refining the bound, focusing more closely on its second term $(1-\pi(\delta))$. This is used in Section~\ref{sec:method} to device an improved algorithm that minimizes simultaneously both terms of the refined error bound.

\subsection{Refined error bound}
\label{sec:new_bound_generalization_error}

Define the indicator random variable $I_\delta(x)$ as follows:
\begin{equation*}
I_\delta(x)=\charac{\{x\in \sX: B_{\delta}(x)\text{ is pure}\}}.
\end{equation*}
By definition, 
\begin{equation*}
\pi\left(\delta\right)=P\left(\left\{ x \in \sX:B_{\delta}\left(x\right)\text{ is pure}\right\} \right)=\sE[I_\delta(x)].
\end{equation*}

Since the distribution of $\sX$ is not known apriori, we use the \emph{empirical distribution} to approximate the expected value of this random variable. Thus, with labeled set $\sL=\left\{x_i\right\}_{i=1}^b$
\begin{equation*}
%\label{eq:mean-approx}
\begin{split}
    \hat{\pi}(\delta) = 
   \hat{\sE}[\charac{\{x\in \sL: B_{\delta}(x)\text{ is pure}\}}] =
    \frac1b\sE[\sum_{i=1}^b\charac{\{ B_{\delta}(x_i)\text{ is pure}\}}]\\ = \frac1b\sum_{i=1}^b\sE[\charac{\{ B_{\delta}(x_i)\text{ is pure}\}}]= \frac1b\sum_{i=1}^b P[ B_{\delta}(x_i)\text{ is pure}].
\end{split}
\end{equation*}
With this approximation 
\begin{equation}
\label{eq:new_bound}
\begin{split}
 \sE \left[f^{N}(x)\neq f(x)\right] \leq\Big[ 1-P[C(\sL,\delta)]\Big] \qquad\qquad  \\ + \Big[1-\frac1b\sum_{i=1}^b P[ B_{\delta}(x_i)\text{ is pure}]\Big]+\varepsilon.
\end{split}
\end{equation}
where $\varepsilon\underset{b\to\infty}{\to}0$ bounds the error of the approximation $\hat{\pi}(\delta)$. 

Note that the refined bound in (\ref{eq:new_bound}) depends on the purity separately at each labeled point in $\sL$. Next, we further refine this bound by allowing the fixed radius $\delta$ to be chosen individually, where $\delta_i$ denotes the radius of point $x_i\in\sL$. Let $\Delta=[\delta_i]_{i=1}^{b}$ denote the list of individual radii corresponding to $\sL=\{x_i\}_{i=1}^{b}$. The cover defined by $\Delta$ is $C(\sL, \Delta) = \bigcup_{(x_i,\delta_i)\in \sL\times\Delta} B_{\delta_i}(x_i)$.

%\inbal{Daphna – I think we shouldn’t change the entire analysis to nNN. Instead, we should just add justification for this transition. If we start with nNN, it will be less intuitive, and we’ll have to justify why ProbCover's analysis is still valid, explain the fine-tuned deltas in their proposed algorithm, and it will become really confusing. I added a reference to the transition justification. What do you think?}
%
With this choice, it is necessary to replace the 1-NN classifier with a suitable variant of nearest neighbor classification - the \emph{normalized Nearest Neighbor} (1-nNN) classifier. In 1-nNN, distances to point $x_i\in\sL$ are normalized by the corresponding $\delta_i$, after which the nearest neighbor $f^{N}(x) = \text{argmin}_{x_i\in\sL} \frac{d(x,x_i)}{\delta_i}$ is computed. This transition preserves the original problem settings, and guarantees that if a labeled point $c\in\sL$ assigns labels to other points, it also covers them (see \app\ref{app:nNN_explanation}). Consequently, if a point $x$ is mislabeled, there exists a $c\in\sL$ that covers it, leading to a reduction in purity. 
The bound on the error in (\ref{eq:new_bound}) becomes
\begin{equation}
\label{eq:local_bound}
\begin{split}
 \sE \left[f^{N}(x)\neq f(x)\right] \leq \Big[ 1-P[C(\sL,\Delta)]\Big] \qquad\qquad  \\+\Big[1-\frac1b\sum_{i=1}^b P[ B_{\delta_i}(x_i)\text{ is pure}]\Big]+\varepsilon.
\end{split}
\end{equation}
This bound describes a trade-off between probability coverage and purity, and guides the algorithm design in Section~\ref{sec:method}. Note that in (\ref{eq:local_bound}), the cover $P[ B_{\delta_i}(x_i)\text{ is pure}]$ is refined locally, using $B_{\delta_i}$ instead of $B_{\delta}$. 
%\inbal{Daphna – one of the reviewers didn’t understand the connection between the bound and the objective function, so I added this explanation. What do you think?} 
\subsection{Competence score $\boldsymbol{\cS}$}
\label{sec:AL-mixed}

In order to obtain a useful competence score\footnote{When the labeled set $\sL$ is large, it may be possible to set aside a validation set to directly estimate the competence of learner $\learner$, but this is not feasible when the labeled set $\sL$ is small.}, which can be effectively used in objective function (\ref{eq:obj_func_def}), we require that it satisfies the following conditions: (i) Lie in the range $[0, 1]$; (ii) Depend on the entire dataset; (iii) Monotonically increase with the competence of learner $\learner$.
% \begin{enumerate}[label=(\roman*)]
%     \item Lie in the range $[0, 1]$. \hfill
%     \item Depend on the entire dataset. \hfill
%     \item Monotonically increase with the competence of learner $\learner$.
% \end{enumerate}
In accordance, we propose to use a constrained Sigmoid function that depends on the probability of coverage:
\begin{equation}
\label{def:competence_score}
%\cS(P[C(\sL,\Delta)]) = \frac{1 + e^{-k(1-a)}}{1 + e^{-k(P[C(\sL,\Delta)]-a)}}
\cS(\sL, \Delta) = \frac{1 + e^{-k(1-a)}}{1 + e^{-k(P[C(\sL,\Delta)]-a)}}
\end{equation}

This score follows the sigmoid curve, where $a$ denotes its midpoint and $k$ the steepness of the curve. In \app\ref{app:hyper_parameters_exp}, we analyze the choice of these parameters and its impact on the algorithm. Note that this score satisfies the conditions stated above for the nNN classifier: (i)-(ii), as it is a function in $[0,1]$ that depends on $P[C(\sL,\Delta)]=\frac{|C(\sL, \Delta)|}{|\sX|}$, while (iii) follows from the error bound (\ref{eq:local_bound}). 

Below, we demonstrate that $\cS$ is monotonically increasing with the budget $b$, which supports the alignment of definition (\ref{def:competence_score}) with empirical evidence from prior work. Specifically, \citet{hacohen2022active} showed both empirically and theoretically that the training budget $b = |\sL|$ serves as a valuable indicator for the relative suitability of ${\mS}(x)$ and ${\mF}(x)$, which are employed in (\ref{eq:obj_func_def}). In our formulation, this relationship informs a trade-off represented by the score $\cS(\sL, \Delta)$ defined in (\ref{def:competence_score}). It remains to be shown that $\cS(\sL, \Delta)$ is indeed monotonically increasing with $b$. 
\begin{proposition}
For two labeled sets $\sL,\sL'$, if $\sL\subseteq \sL'$ then   $\cS(\sL, \Delta)\leq \cS(\sL',\Delta')$.
\end{proposition}
\begin{proof}
First, we note that
\begin{equation*}
\begin{split}
\sL\subseteq \sL' &\implies C(\sL,\Delta) = \bigcup_{x_i\in \sL} B_{\delta_i}(x_i) \\&\subseteq  \left(\bigcup_{x_i\in \sL} B_{\delta_i}(x_i)\right) \cup \left(\bigcup_{x_i\in \sL'\setminus\sL} B_{\delta_i}(x_i)\right) \\&= C(\sL',\Delta')
\end{split}
\end{equation*}
Since probability is monotonic, it follows that $P[C(\sL,\Delta)]\leq P[C(\sL',\Delta')]$.
From (\ref{def:competence_score}), and by the definition of the logistic function using the provided parameters limits, $\cS(\sL, \Delta)$ is monotonically increasing in $\sL$, which implies that $\cS(\sL, \Delta)\leq \cS(\sL',\Delta')$.
\end{proof}

\section{Our method: \MethodName}
\label{sec:method}
The error bound presented in (\ref{eq:local_bound}), derived in Section~\ref{sec:theory} by our theoretical analysis, highlights the trade-off between maximizing coverage $P[C(\sL,\Delta)]$ and maintaining high purity $P[\cup_i \{B_{\delta_i}(x_i) \text{ is pure}\}]$ while seeking an optimal query set. The analysis also highlights the benefit of using a dynamic competence score to adjust the algorithm’s objective function based on coverage. Below, we incorporate these principles into a new active learning strategy called \textbf{\MethodName}, a unified active learning strategy designed for all budget scenarios. Notably, since \citep{yehuda2022active} demonstrated that maximizing coverage is an NP-hard problem, we address this challenge with a greedy algorithm for query selection.

\begin{figure*}[htb]
    \centering
  \begin{subfigure}{0.49\linewidth}
    \centering
 \includegraphics[width=\linewidth]{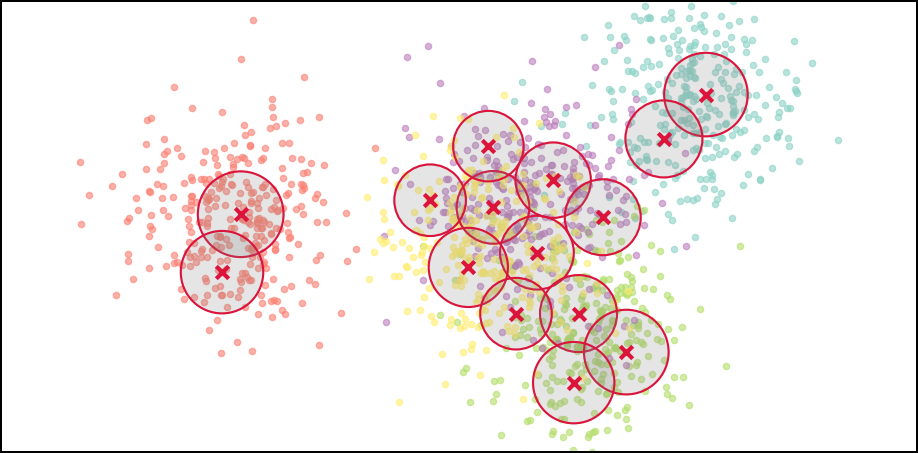}
    \caption{Initial active sampling using \(\delta_0\) in DCoM, following \probcover\ approach.}
    \label{fig:updating_delta:a}
  \end{subfigure}
  \begin{subfigure}{0.49\linewidth}
    \centering
    \includegraphics[width=\linewidth]{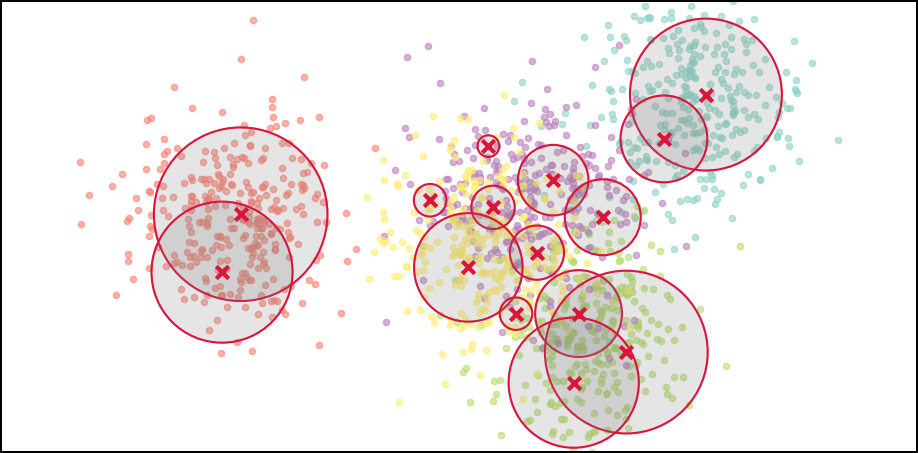}
    \caption{\MethodName's updated $\Delta$ list, which includes a suitable $\delta$ value for each example in $\sL$}
    \label{fig:updating_delta:b}
  \end{subfigure}
\caption{Illustration of \textit{Step 4: $\Delta$-Adjustment} in \MethodName, comparing the coverage obtained by \MethodName\ (b) with \probcover\ (a) using budget $b=15$. Selected points are marked by x, and the covered area is shaded in light gray. While \probcover\ maintains a constant $\delta$ for each example, \MethodName\ seeks to maximize coverage while preserving the purity of the balls by fitting a suitable radius for each center.}
  \label{fig:updating_delta}
\end{figure*} 

\subsection{Framework and definitions}
$\forall x \in \sX$, let $M(x)$ denote the normalized margin between the two highest softmax outputs from the last trained model $\sM$. Define $\Delta$ as the set of individual $\delta$-values for each labeled example in $\sL$. As in the theoretical analysis, $P[C(\sL,\Delta)]$ denotes the probability of coverage for $\sL$ given $\Delta$. The competence score $\cS(\sL, \Delta)$ is defined by the objective function (\ref{def:competence_score}). The graph $\{G_{\delta} = (V, E)\}$ is a directed adjacency graph where nodes are samples, and edges connect pairs of nodes if their distance in the embedding space is less than $\delta$.

\subsection{Active learning method}
\label{sec:greedy_alg}
Active Learning (AL) is performed iteratively (see Fig.~\ref{fig:AL_illustration}): at each iteration $i$, a query set $\tQ$ of $q$ unlabeled examples is selected based on a specific strategy, where $q$ denotes the query set size and $\tQ$ denotes the current query set. These examples are then labeled by an oracle, added to the labeled set $\sL$, and removed from the unlabeled pool $\sU$. Subsequently, the algorithm parameters may be updated as needed, depending on the specific requirements of the algorithm. This process is repeated until the label budget is exhausted or predefined termination conditions are met. 

Our AL method \textbf{\MethodName} involves several distinct steps:

\myparagraph{Step 1: Initialization.} Identify an appropriate embedding space for $\sU \cup \sL$, where distances are expected to be inversely correlated to semantic similarity. This can be achieved using self-supervised or representation learning techniques. Select an initial radius $\delta_0$ as outlined in \citep{yehuda2022active}, and initialize the parameter $\Delta=\emptyset$. If $\sL\neq\emptyset$, skip \textit{Step 2} for now and begin the iterations with \textit{Step 3}.

\myparagraph{Step 2: Active sampling.} Select $q$ examples from $\sU$ for labeling (see Alg~\ref{alg:active_sampling} for the pseudo-code). Begin by computing the normalized margin $M(x)$ $\forall x \in \sU$ using the latest model $\sM$. Next, construct a directed adjacency graph $\{G_{\delta_\text{avg}} = (V, E)\}$, where $\delta_\text{avg}$ is the average of $\Delta$, or $\delta_0$ if $\Delta$ is empty. To avoid covering points multiple times, $\forall x_i \in \sL$,  $\forall x \in B_{\delta_i}(x_i)$ and  $\forall v\in{V}$, prune all incoming edges $e = (v, x) \in E$. Additionally, prune all outgoing edges $e = (x_i, v) \in E$.

Using this graph, compute the Out-Degree-Rank (ODR) $\forall x \in \sU$. Likewise, compute the probability of coverage $P[C(\sL,\Delta)]$ of the current $\sL$ and its corresponding $\Delta$, and use these quantities to determine $\cS(\sL, \Delta)$. Now select $q$ points from $\sU$ in a greedy manner as follows:
\begin{enumerate}[leftmargin=0.7cm,nolistsep]
\setlength\itemsep{.3em}
    \item $\forall x\in\sU$, compute the Out-Degree-Rank $\text{ODR}(x)$. 
    \item $\forall x\in\sU$, compute a ranking score 
\begin{equation*}    
\hspace{-.6cm} R(x) = \cS(\sL, \Delta) \cdot (1-M(x)) + (1 - \cS(\sL, \Delta)) \cdot \text{ODR}(x).
\end{equation*}    
    \item Select the vertex $x_{\text{max}}$ with the highest score
\begin{equation*}    
x_{\text{max}}=\text{argmax}_{x\in\sU} R(x).
\end{equation*}    
    \item Remove all incoming edges to $x_{\text{max}}$ and its neighbors, implying that $ODR(x_\text{max})= 0$.
    \item Set $M(x_{\text{max}}) = 1$.
\end{enumerate}

\begin{algorithm}[htb]
    \footnotesize
    \caption{\MethodName, \textbf{Active sampling}}
    \label{alg:active_sampling}

    \textbf{Input}: unlabeled pool $\sU$, labeled pool $\sL$, query size $q$, list of ball sizes $\Delta$, trained model $\sM$ if $\sL\neq\emptyset$\\
    \textbf{Output}: a set of points to query, and the coverage of $\sL$

    \begin{algorithmic}[1]
    \STATE $\forall x\in \sU$, $M(x) \leftarrow$ normalized margin between the two highest softmax outputs of $\sM$ if $\sL\neq\emptyset$, 1 otherwise
%    \STATE \daphna{don't you need to initialize $\Delta$ if $\sL=\emptyset$?}
    \STATE Compute $P[C(\sL,\Delta)]$
    \STATE Compute $\cS(\sL, \Delta)$
    \STATE $\delta_{\text{avg}} \leftarrow$ Average($\Delta$)
    \STATE $G = (V = \sU\cup\sL, E = \{(x, x'): x' \in B_{\delta_{\text{avg}}}(x)\})$
    \FOR{$(x_i, \delta_i) \in \sL\times\Delta$}
        \STATE $\forall x\in B_{\delta_i}(x_i)$: $\forall e=(v,x)\in E$, remove $e$   
        % {\color{airforceblue} \COMMENT{DON'T COVER POINTS TWICE}}
        \STATE $\forall e=(x_i,v)\in E$, remove $e$ 
    \ENDFOR
    % \FOR{$(x_i, \delta_i) \in \sL\times\Delta$}
    %     \STATE $\forall x\in B_{\delta_i}(x_i)$: $\forall e=(v,x)\in E$, if $v\in B_{\delta_i}(x_i)$ then remove $e$ \daphna{tried to fix it but I think it's not precise - check all quantifiers}
    %     \STATE $\forall e=(x_i,v)\in E$, remove $e$
    % \ENDFOR
    \vspace{\baselineskip}
    \STATE ${\tQ} \leftarrow \emptyset$
    \FOR{$i=1,...,q$}
        \STATE $\forall x\in \sU$, $\text{ODR}(x)\leftarrow $Out-Degree Rank
        \STATE $\forall x\in \sU$, $R(x) \leftarrow \cS(\sL, \Delta) \cdot (1-M(x))$ + (1 - $\cS(\sL, \Delta))\cdot\text{ODR}(x)$
        \STATE $x_{max} \leftarrow\{\text{argmax}_{x\in\sU} R(x)\}$
        \STATE $\tQ\leftarrow {\tQ}\cup \{x_{max}\}$
        \STATE $\forall x\in B_{\delta_{avg}}(x_{max})$: $\forall e=(v,x)\in E$, remove $e$ 
        \STATE $M(x_{max})\leftarrow\text{1}$ 
    \ENDFOR
    \RETURN ${\tQ}$, $P[C(\sL,\Delta)]$
    \end{algorithmic}
\end{algorithm}

\myparagraph{Step 3: Model training:}
Obtain labels for the query set $\tQ$ derived in \textit{Step 2}. Remove $\tQ$ from $\sU$ and add it to $\sL$. Finally, train model $\sM$ using the updated $\sL$ (and also $\sU$ if the learner employs semi-supervised learning).

\myparagraph{Step 4: $\boldsymbol{\Delta}$-Adjustment.} Determine a suitable $\delta_i$ for each new labeled example $x_i \in \tQ$ and add it to $\Delta$, see Alg~\ref{alg:Delta_Adjustment} for pseudo-code and illustration in Fig.~\ref{fig:updating_delta}. Begin by using the updated model $\sM$ to predict labels for all points in the unlabeled set $\sU$. Likewise, compute a purity threshold $\tau = m \cdot P[C(\sL \setminus \tQ, \Delta)] + d$, where\footnote{The linear transformation defined by $m,d$ adjusts the range of coverage scores  into the relevant range of purity scores $[b, b+m]$.} $P[C(\sL \setminus \tQ, \Delta)]$, the probability of coverage using the old labeled set, is derived in \textit{Step 2}. Finally, for each new labeled example $v \in \tQ$, search for the largest radius $\delta_{\text{opt}}$ whose purity $P\left(\left\{ x \in B_{\delta_{\text{opt}}}\left(v\right): f(x) = f(v) \right\}\right)$ remains above $\tau$. Assuming that $\pi(B_{\delta}(x))$ is monotonic with $\delta$, use binary search. 
%\inbal{Daphna - We considered writing the linear function as a normalization to provide that intuition (instead of using a constant value, we allow it to be slightly modified according to the model's coverage/competence). However, I was unable to find a good way to integrate it.}

\begin{algorithm}[htb]   
    \footnotesize
    \caption{\MethodName, $\boldsymbol\Delta$\textbf{-Adjustment}}
    \label{alg:Delta_Adjustment}

    \textbf{Input}: unlabeled pool $\sU$, labeled pool $\sL$, query pool $\tQ$, maximal ball size $\delta_{\text{max}}$, list of ball sizes $\Delta$ for $\sL\setminus{\tQ}$, trained model $\sM$ from the current iteration and $P[C(\sL\setminus{\tQ},\Delta)]$ from \textit{Step 2}\\
    \textbf{Output}: updated list of ball sizes $\Delta$

    \begin{algorithmic}[1]    
    \STATE $\tau \leftarrow m \cdot P[C(\sL\setminus{\tQ},\Delta)] + d$
    \STATE $\hat{\sY}\leftarrow$ The predicted label for each $x\in\sU$ using the model $\sM$
    \FOR{$v \in {\tQ}$}
        \STATE $\delta_{\text{opt}}\leftarrow \text{argmax}_\delta [P\left(\left\{ x \in B_{\delta}\left(v\right):f(x)=f(v)\right\} \right)>\tau$]
        %Apply binary search over $\delta$ values between $0$ and $\delta_{\text{max}}$ to find the greatest $\delta_{\text{opt}}$ that brings the closest purity value to $\tau$. Calculate the purity using $\hat{\sY}$.
        \STATE $\Delta \leftarrow \Delta + [\delta_{\text{opt}}]$
    \ENDFOR
    
    \RETURN $\Delta$
    \end{algorithmic}
\end{algorithm}

%The algorithm relies on $P[C(\sL,\Delta)]$ as an indicator of progress, which adjusts the purity threshold and affect the transition from a density-based scoring method to an margin-based one. 

\section{Empirical evaluation}
\label{sec:results}

\subsection{Methodology}
\label{sec:methodology}
This section investigates two deep AL frameworks: a \emph{fully supervised} approach, wherein a deep model is trained exclusively on the annotated dataset $\sL$ as a standard supervised task, and a \emph{semi-supervised} framework trained on both the annotated and unlabeled datasets, $\sL$ and $\sU$. Our experimental setup is based on the codebase of \citep{munjal2022towards}, ensuring a fair comparison among the various AL strategies by using the same network architectures while sharing all the relevant experimental conditions. 

More specifically, we trained ResNet-18 \citep{DBLP:conf/cvpr/HeZRS16} on the following benchmark datasets: STL-10 \citep{coates2011analysis}, SVHN \citep{netzer2011reading}, CIFAR-100 \citep{krizhevsky2009learning} and subsets of ImageNet \citep{deng2009imagenet} including ImageNet-50, ImageNet-100, and ImageNet-200 as defined in \citep{van2020scan}. The hyper-parameters are detailed in \app\ref{app:impl_details}. In the semi-supervised framework, we used FlexMatch \citep{zhang2021flexmatch} and the code provided by \citep{usb2022}, adopting the parameters specified by \citep{zhang2021flexmatch} (see details in \app\ref{app:impl_details}). While the ResNet-18 architecture may no longer achieve state-of-the-art results on the datasets examined, it serves as a suitable platform to evaluate the efficacy of AL strategies in a competitive and fair environment, where they have demonstrated benefits. 

In our comparisons, we used several active learning strategies for selecting a query set:
\begin{inparaenum}[(i)] \item \emph{Random} sampling, \item \emph{Min margin} — lowest margin between the two highest softmax outputs; \item \emph{Max entropy} — highest entropy of softmax outputs; \item \emph{Uncertainty} — 1 minus the highest softmax output; \item \emph{DBAL} \citep{gal2017deep}; \item \emph{CoreSet} \citep{sener2018active}; \item \emph{BALD} \citep{kirsch2019batchbald}; \item \emph{BADGE} \citep{ash2019deep}; \item \emph{ProbCover} \citep{yehuda2022active}; and \item \emph{LDM-s} \citep{cho2023querying}. \end{inparaenum} When available, we used the code provided in \citep{munjal2022towards} for each strategy. When it was unavailable, we used the code from the repositories of the corresponding papers. Since \citep{cho2023querying} doesn't provide code, we compared our results with the corresponding available results reported in their paper, despite differences in the running setup (see \app\ref{app:comp_per_dataset}).
%In our comparisons, we used several active learning strategies (or optimizaiton criteria) for the selection of a query set:
% \begin{inparaenum}[(1)] \item \emph{Random} sampling, \item \emph{Min margin} -- lowest margin between the two highest softmax outputs, \item \emph{Max entropy}  --  highest entropy of softmax outputs \citep{shannon1948A}, \item \emph{Uncertainty} -- 1 minus maximum softmax output, \item \emph{DBAL} \citep{gal2017deep}, \item \emph{CoreSet} \citep{sener2018active}, \item \emph{BALD} \citep{kirsch2019batchbald}, \item \emph{BADGE} \citep{ash2019deep}, \item \emph{ProbCover} \citep{yehuda2022active}, \item \emph{LDM-s} \citep{cho2023querying}. \end{inparaenum} When available, we used for each strategy the code provided in \citep{munjal2022towards}. When it was unavailable, we used the code from the repositories of the corresponding papers. \citep{cho2023querying} doesn't provide code, so we compare our results to those reported in their paper, despite differences in the running setup.

\MethodName\  and \emph{ProbCover} require a suitable data embedding. For STL-10 , SVHN, CIFAR-10 and CIFAR-100 we employed the SimCLR \citep{chen2020simple} embedding. For the ImageNet subsets we used the DINOv2 \citep{oquab2023dinov2} embedding. Section~\ref{sec:ablation} provides an extensive ablation study, which includes alternative embedding spaces. %Moreover, within this section, we compare multiple uncertainty-based methods, revealing that the high-budget algorithm doesn't exert significant influence when initialized with the dynamic probability coverage method. Additionally, we demonstrate the efficacy of blending different methods by decomposing the algorithm into its two objective functions.
In all experiments, identical settings and hyper-parameters are used  as detailed in \app\ref{app:impl_details}. The robustness of our method to these choices is discussed in \app\ref{app:hyper_parameters_exp}. 

Using a sparse representation of the adjacency graph enables \MethodName\ to handle large datasets efficiently without exhaustive space requirements. The algorithm's complexity, including adjacency graph construction and sample selection, is discussed in \app\ref{app:complexity_analysis}.

\subsection{Main results}
\label{sec:main_results}

% with DCoM
\begin{figure}[htb] % htb
\center
\begin{subfigure}{\columnwidth}
\includegraphics[width=\linewidth]{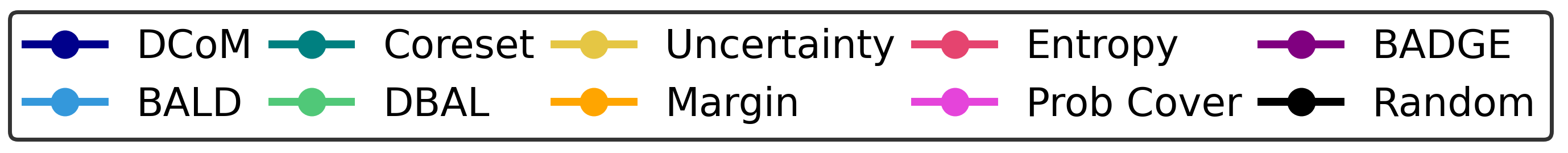}
\end{subfigure}
%\\
\begin{subfigure}{.49\columnwidth}
\includegraphics[width=\linewidth]{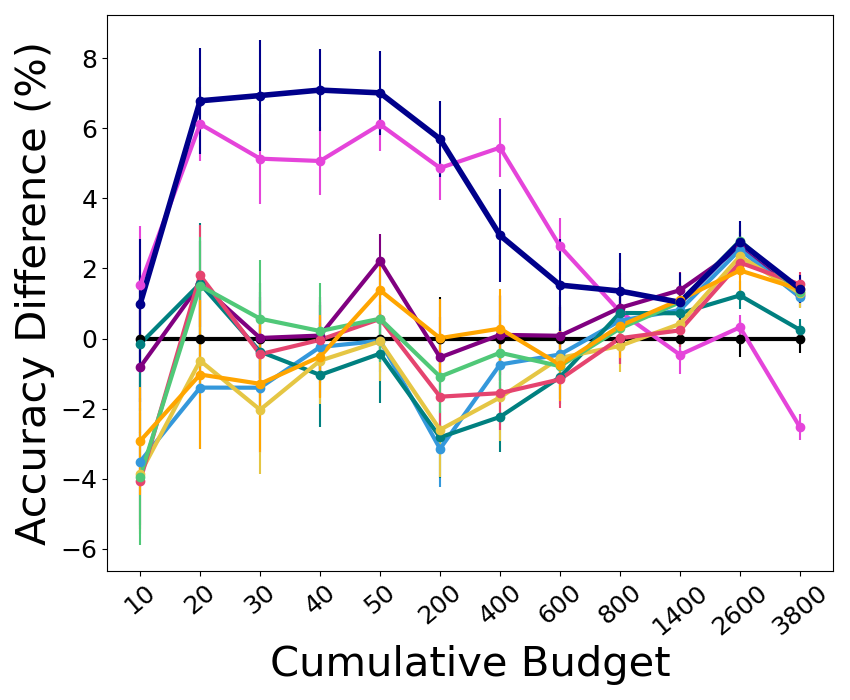}
\caption{STL-10} 
\label{fig:main_results:a}
\end{subfigure}
\begin{subfigure}{.49\columnwidth}
\includegraphics[width=\linewidth]{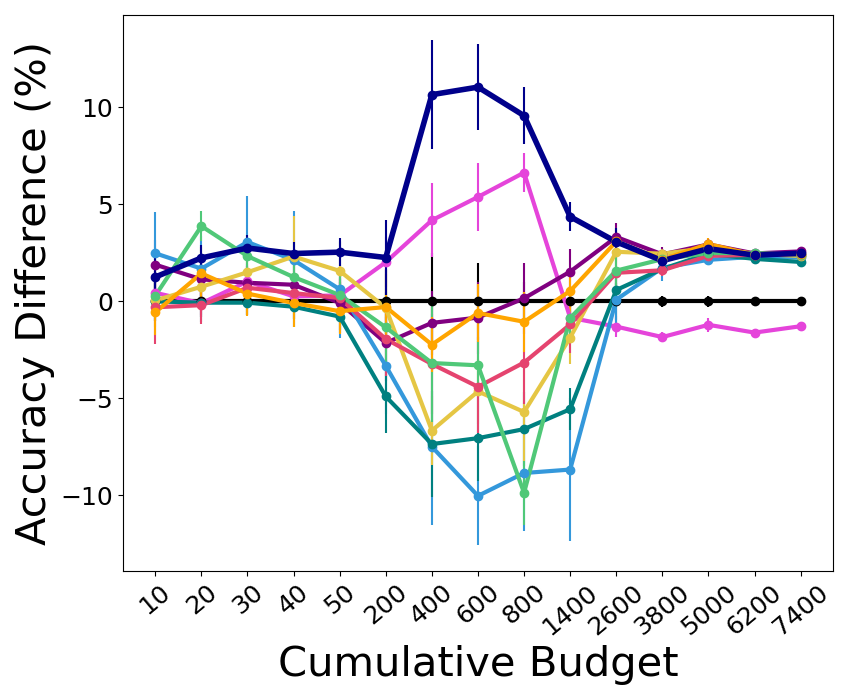}
\caption{SVHN} 
\label{fig:main_results:b}
\end{subfigure}
% \begin{subfigure}{.325\textwidth}
% \includegraphics[width=\linewidth]{results/results_05_13/DCoM_wide/CIFAR10/2.png}
% \caption{CIFAR-10}
% \label{fig:main_results:c}
% \end{subfigure}
%\\
\begin{subfigure}{.49\columnwidth}
\includegraphics[width=\linewidth]{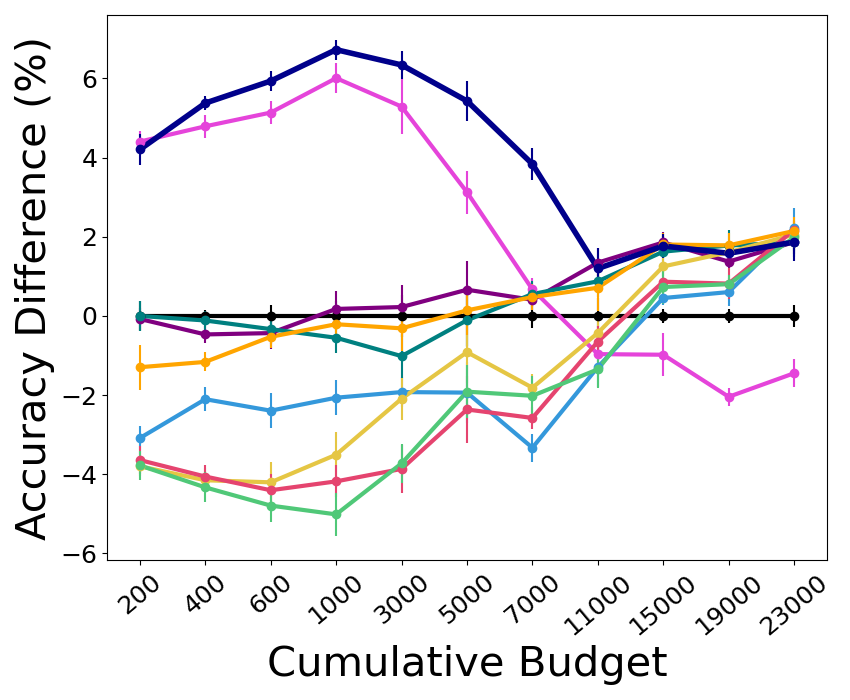}
\caption{CIFAR-100} 
\label{fig:main_results:d}
\end{subfigure}
\begin{subfigure}{.49\columnwidth}
\includegraphics[width=\linewidth]{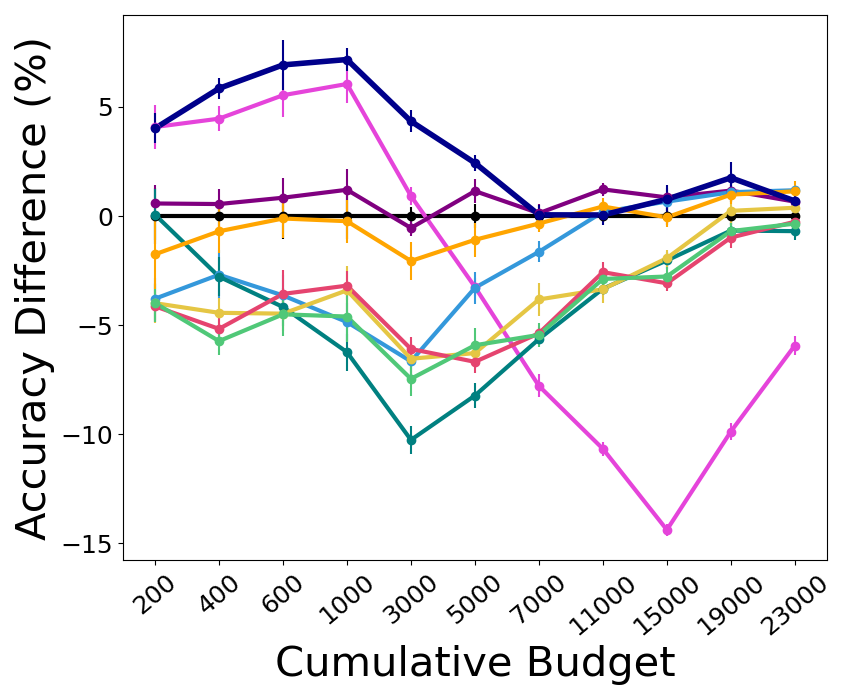}
\caption{ImageNet-50} 
\label{fig:main_results:e}
\end{subfigure}
\begin{subfigure}{.49\columnwidth}
\includegraphics[width=\linewidth]{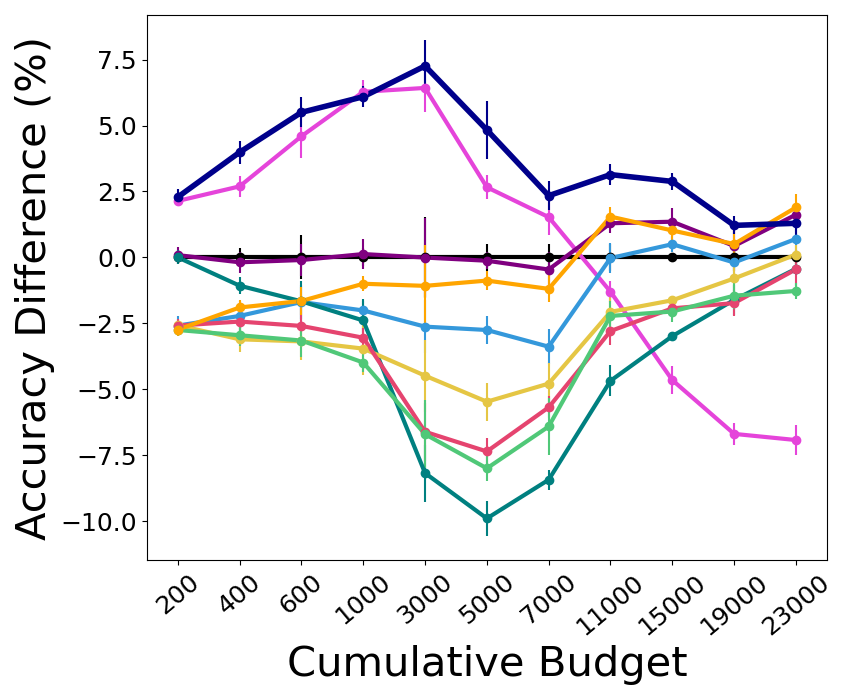}
\caption{ImageNet-100}
\label{fig:main_results:f}
\end{subfigure}
\begin{subfigure}{.49\columnwidth}
\includegraphics[width=\linewidth]{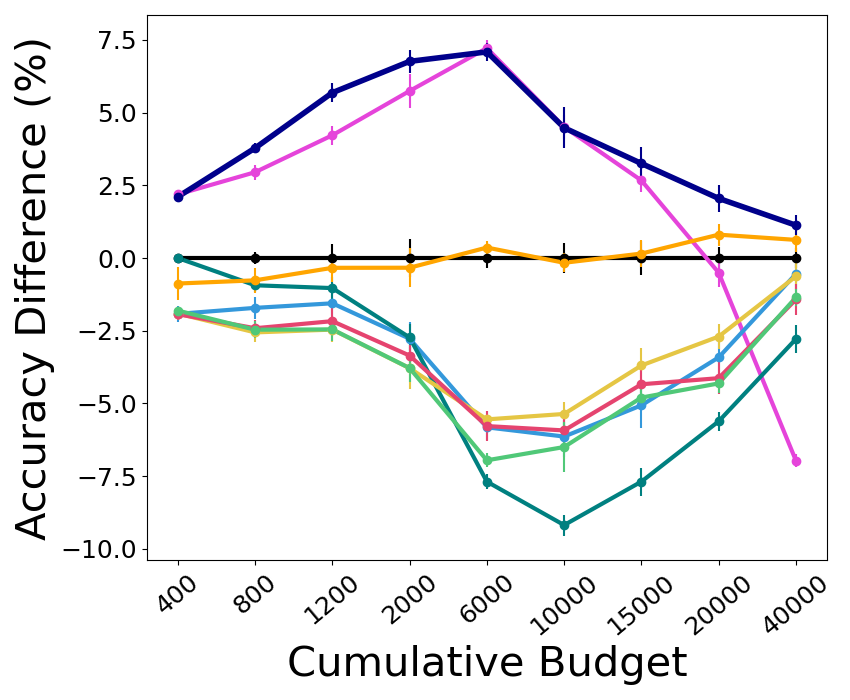}
\caption{ImageNet-200} 
\label{fig:main_results:g}
\end{subfigure}
\caption{Mean accuracy difference between AL algorithms and random queries. Positive mean difference indicates that the corresponding AL method is beneficial. The errors bars correspond to standard error, using 5-10 repetitions (3 for the ImageNet subsets). While some methods perform well only in specific budget regime, \MethodName\ consistently achieves the best results across all budgets. Cumulative budget is evenly spaced for easy range comparison. %\daphna{rewrite last sentence}
}

\label{fig:main_results}
\end{figure}

\myparagraph{Fully supervised framework.} 
We evaluate various AL methods by training a ResNet-18 model from scratch for each different AL strategy, using the labels of the queries selected by the respective method. In each AL iteration, a new model is trained from scratch using the updated labeled set $\sL$. This process is repeated for several AL iterations. Fig.~\ref{fig:main_results} shows the difference between the accuracy obtained by each AL method, and the accuracy obtained when training a similar network with a random query set of the same size (the baseline of no active learning), %. We show mean and the standard error over $5-10$ repetitions ($3$ for the ImageNet's subsets), 
for each AL iteration. The final accuracy of each method is shown in Table~\ref{tab:cifar100_results} for the CIFAR-100 dataset, as mean $\pm$ Standard Error (STE). 
Similar tables for the remaining datasets are shown in \app\ref{app:add_emp_results}.

\begin{table*}[htb]
\centering
\scriptsize
\setlength{\tabcolsep}{1pt} % Adjust horizontal space between cells
\begin{tabularx}{0.995\textwidth}{l*{11}{>{\centering\arraybackslash}X} r}
\toprule
$|\sL|$ & Random & Prob Cover & BADGE & BALD & Coreset & Uncertainty & Entropy & DBAL & Margin & LDM-s & DCoM \\ 
\midrule
200 & 6.39±0.19 & \textbf{10.79±0.08} & 6.31±0.09 & 3.30±0.11 & 6.39±0.19 & 2.59±0.05 & 2.74±0.17 & 2.61±0.18 & 5.09±0.37 & - & \textbf{10.59±0.19} \\ 
400 & 8.74±0.07 & 13.53±0.22 & 8.27±0.14 & 6.63±0.23 & 8.62±0.17 & 4.59±0.23 & 4.68±0.23 & 4.41±0.30 & 7.58±0.17 & - & \textbf{14.11±0.11} \\ 
600 & 10.73±0.13 & 15.86±0.16 & 10.29±0.27 & 8.33±0.31 & 10.39±0.25 & 6.52±0.38 & 6.32±0.27 & 5.94±0.28 & 10.20±0.16 & - & \textbf{16.66±0.11} \\ 
1,000 & 13.49±0.18 & 19.49±0.20 & 13.66±0.26 & 11.42±0.26 & 12.94±0.21 & 9.98±0.38 & 9.31±0.22 & 8.47±0.35 & 13.28±0.19 & - & \textbf{20.21±0.07} \\ 
3,000 & 24.02±0.27 & 29.31±0.42 & 24.24±0.28 & 22.10±0.11 & 23.00±0.28 & 21.93±0.26 & 20.15±0.33 & 20.29±0.23 & 23.71±0.31 & - & \textbf{30.36±0.08} \\ 
5,000 & 31.52±0.39 & 34.64±0.15 & 32.18±0.33 & 29.58±0.38 & 31.41±0.26 & 30.60±0.40 & 29.15±0.46 & 29.60±0.28 & 31.66±0.23 & - & \textbf{36.95±0.10} \\ 
7,000 & 37.69±0.16 & 38.36±0.13 & 38.09±0.10 & 34.36±0.20 & 38.24±0.16 & 35.87±0.19 & 35.11±0.11 & 35.67±0.35 & 38.16±0.17 & 31.85±0.00 & \textbf{41.52±0.25} \\ 
11,000 & 45.66±0.22 & 44.69±0.20 & \textbf{47.00±0.16} & 44.37±0.28 & 46.53±0.18 & 45.23±0.20 & 45.01±0.17 & 44.31±0.24 & 46.37±0.36 & 40.88±0.01 & \textbf{46.86±0.28} \\ 
15,000 & 50.57±0.09 & 49.58±0.46 & \textbf{52.42±0.17} & 51.02±0.09 & \textbf{52.19±0.29} & 51.82±0.16 & 51.43±0.21 & 51.30±0.35 & \textbf{52.37±0.19} & 46.59±0.01 & \textbf{52.33±0.22} \\ 
19,000 & 54.95±0.09 & 52.90±0.14 & 56.32±0.07 & 55.55±0.27 & \textbf{56.72±0.31} & \textbf{56.55±0.21} & 55.76±0.31 & 55.75±0.14 & \textbf{56.73±0.24} & 49.41±0.01 & \textbf{56.53±0.07} \\ 
23,000 & 57.74±0.14 & 56.29±0.21 & \textbf{59.60±0.16} & \textbf{59.96±0.35} & \textbf{59.59±0.22} & \textbf{59.79±0.23} & \textbf{59.91±0.08} & \textbf{59.72±0.21} & \textbf{59.88±0.21} & 52.48±0.00 & \textbf{59.60±0.33} \\ 
\bottomrule
\end{tabularx}
\caption{Model accuracy with varying labeled set sizes $\sL$ (row) and AL strategies (column) while training directly on CIFAR-100 images.}
\label{tab:cifar100_results}
\end{table*}

\myparagraph{Semi-supervised framework.} We evaluate the effectiveness of 5 representative AL strategies with FlexMatch, a competitive semi-supervised learning method. We also evaluate our method, and the results of random sampling (no AL) as baseline. Fig.~\ref{fig:semi_supervised_low_budgets} shows the mean accuracy and STE based on 3 repetitions of FlexMatch, employing labeled sets generated by the different AL strategies. Note that in these experiments, since the query is selected from an unlabeled set, the resulting labeled set $\sL$ may be unbalanced between the classes, which is especially harmful when the budget is small. This affects all the methods. As a result, the accuracy with random sampling may differ from reported results. Notably, \MethodName\ demonstrates significant superiority over random sampling.

\begin{figure}[htb]
    \centering
\includegraphics[width=1\linewidth]{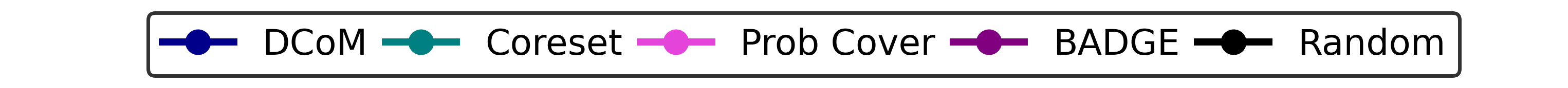}
  \begin{subfigure}{0.49\linewidth}
    \centering
    \includegraphics[width=\linewidth]{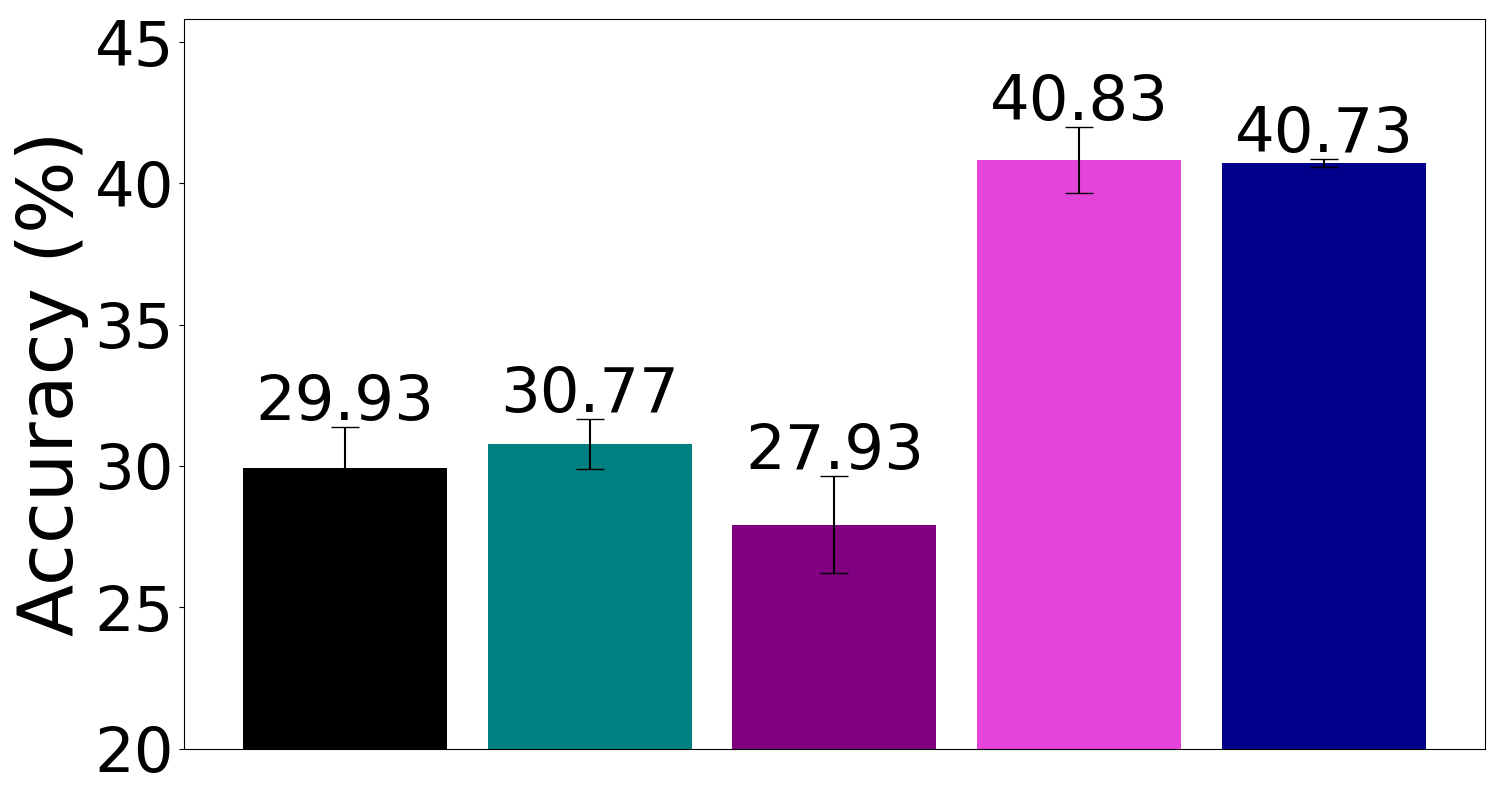}
    \begin{center} {\small $|\sL|=400$} \end{center}
  \end{subfigure}\hfill
  \begin{subfigure}{0.49\linewidth}
    \centering
    \includegraphics[width=\linewidth]{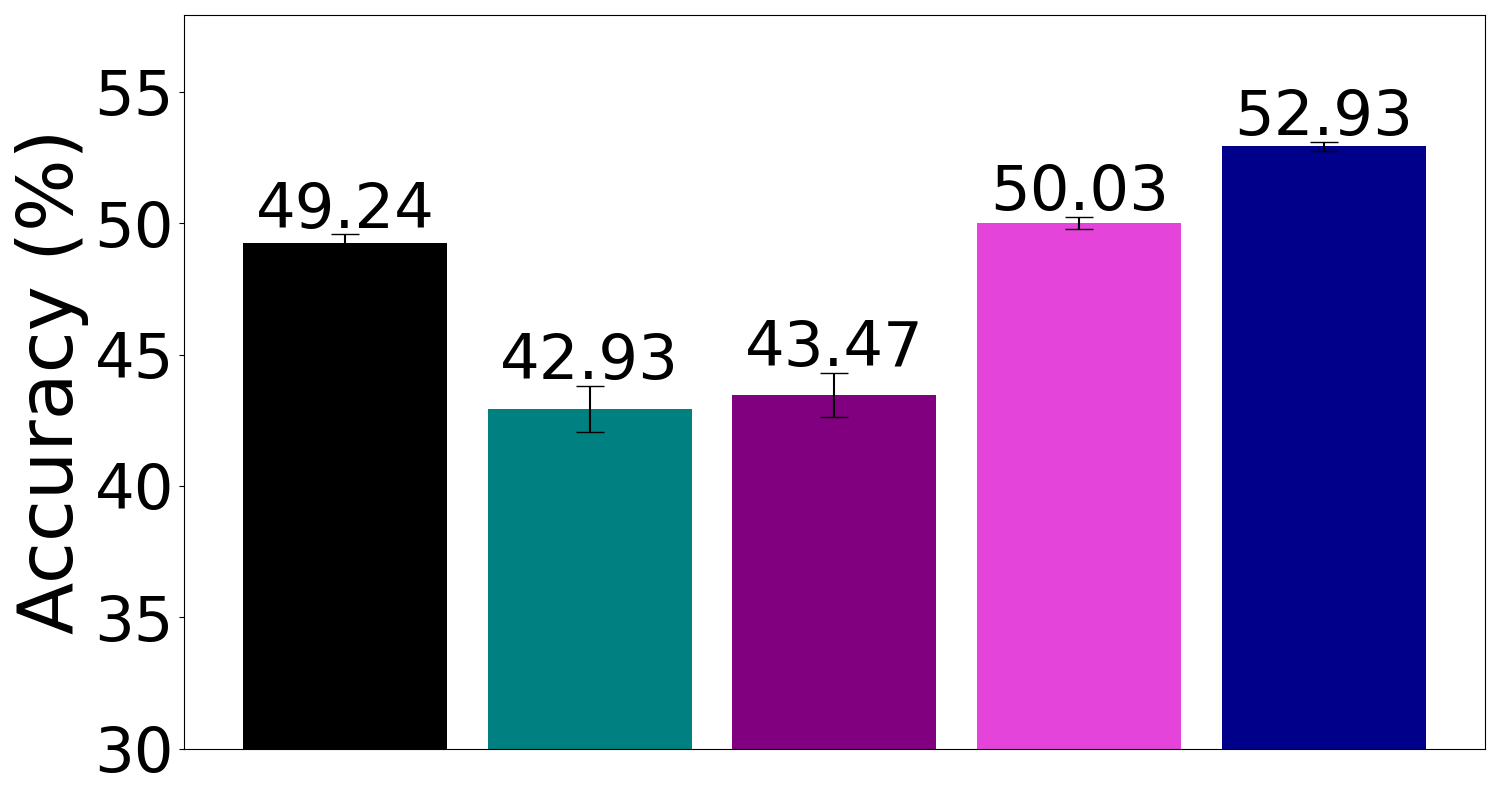}
    \begin{center} {\small $|\sL|=1,000$} \end{center}
  \end{subfigure}
  \caption{Comparison of AL strategies in a semi-supervised learning framework. Each bar shows the mean test accuracy over 3 repetitions of FlexMatch, trained with 400 and 1,000 labeled examples on CIFAR-100. Three AL rounds are conducted, selecting 100, 300, and 600 examples in each round.}
\label{fig:semi_supervised_low_budgets}
\end{figure}

\subsection{Ablation study}
\label{sec:ablation}
We present a series of experiments where we investigate the choices made by our method, including different feature spaces and different uncertainty-based measures. The results demonstrate the robustness of our method to these choices. Additional choices are investigate in \app\ref{app:add_emp_results}, including the contribution of a dynamic $\delta$ and initialization $\delta_0$ %values for each example individually (\app\ref{app:Delta_distribution}), and robustness to $\delta_0$ initialization 
(\app\ref{app:delta_initialization}). In \app\ref{app:comp_score_illust}, we evaluate the effectiveness of $\cS(\sL, \Delta)$  from (\ref{def:competence_score}) in capturing competence. 

\begin{figure}[htb]
    \centering
\includegraphics[width=.975\linewidth]{results/two_rows.png}
  \begin{subfigure}{0.49\linewidth}
    \centering
    \includegraphics[width=\linewidth]{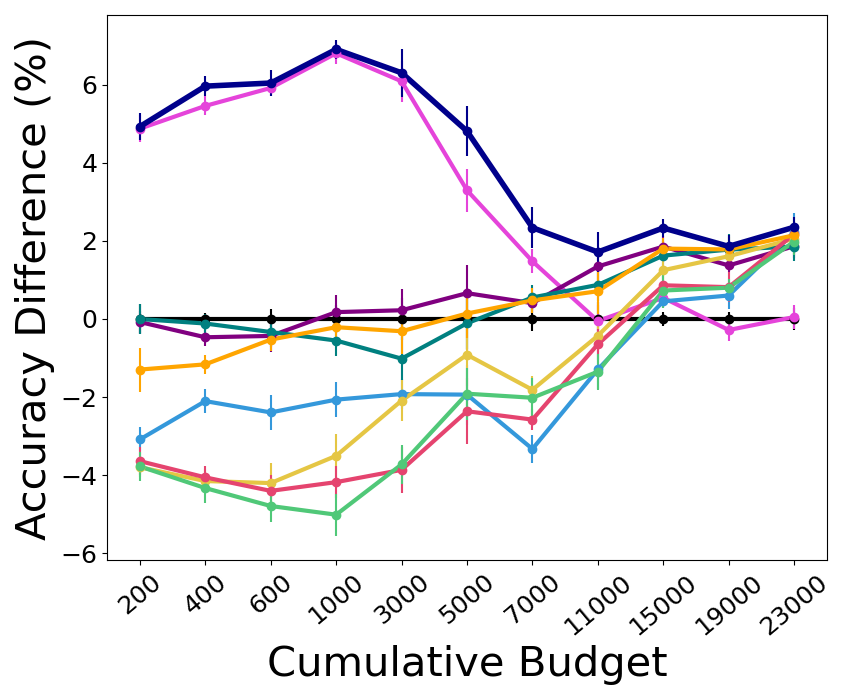}
    \caption{MOCOv2+}
    \label{fig:diff_feature_spaces_cifar100:a}
  \end{subfigure}\hfill
  \begin{subfigure}{0.49\linewidth}
    \centering
    \includegraphics[width=\linewidth]{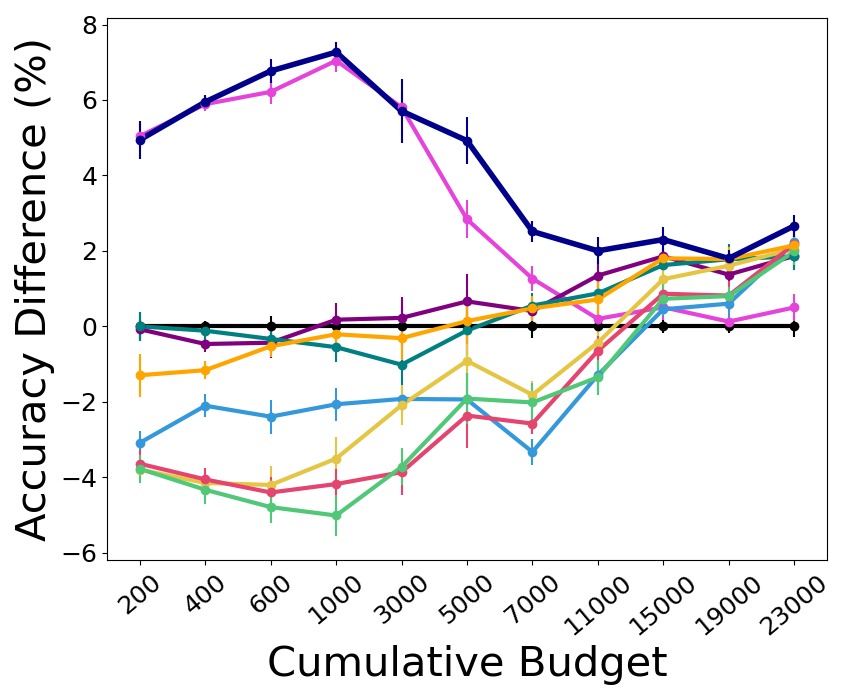}
    \caption{Barlow Twins}
    \label{fig:diff_feature_spaces_cifar100:c}
  \end{subfigure}
  \caption{Performance over CIFAR-100, using two different embedding spaces for \MethodName\ and \emph{ProbCover}, see details in the caption of Fig.~\ref{fig:main_results}. Clearly, \MethodName\ consistently achieves the best results. Results using an additional embedding space are shown in \app\ref{app:different_emb_space}.}
  \label{fig:diff_feature_spaces_cifar100}
\end{figure}

\myparagraph{Different feature spaces.}
As discussed in Section~\ref{sec:method}, our approach relies on the existence of a suitable embedding space, where distance is indicative of semantic dissimilarity. We now repeat the basic fully-supervised experiments while varying the embedding space, comparing various popular self-supervised representations. Results are reported in Fig.~\ref{fig:diff_feature_spaces_cifar100}, showing that \MethodName\ consistently delivers the best performance regardless of the specific embedding used.

\myparagraph{Breaking down the algorithm into its components.}
The objective function in (\ref{eq:obj_func_def}) includes the linear combination of two objective functions, weighed by competence score $\cS(\sL, \Delta)$ from (\ref{def:competence_score}). 
% The general form of this objective function is an important part of our method, and we therefore 
We evaluate the contribution of each component in this objective function separately. Results are shown in Fig.~\ref{fig:break_down_fig}, wherein \emph{DPC} denotes the Dynamic Probability Coverage, corresponding to ${\mS}(x)$ in \MethodName, and \emph{Margin} denotes \MethodName's choice for ${\mF}(x)$. Clearly, the weighted objective function of \MethodName\ yields superior results% over its components across all labeled set sizes $\vert\sL\vert$
. 
% Similar behavior is observed for several other datasets and other selections of ${\mF}(x)$, see \app\ref{app:break_down_with_entropy}.

\begin{figure}[htb]
    \centering
\includegraphics[width=.8\linewidth]{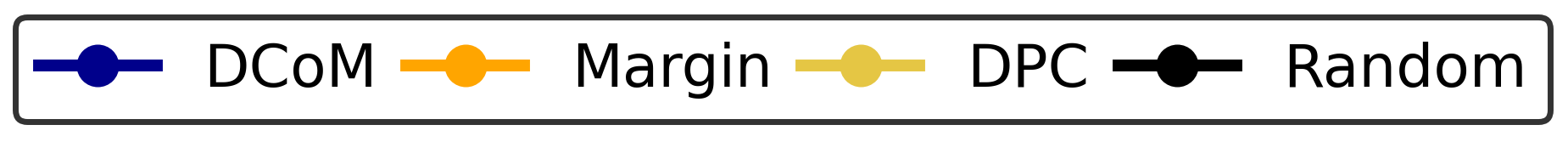}
  \begin{subfigure}{0.49\linewidth}
    \centering
    \includegraphics[width=\linewidth]{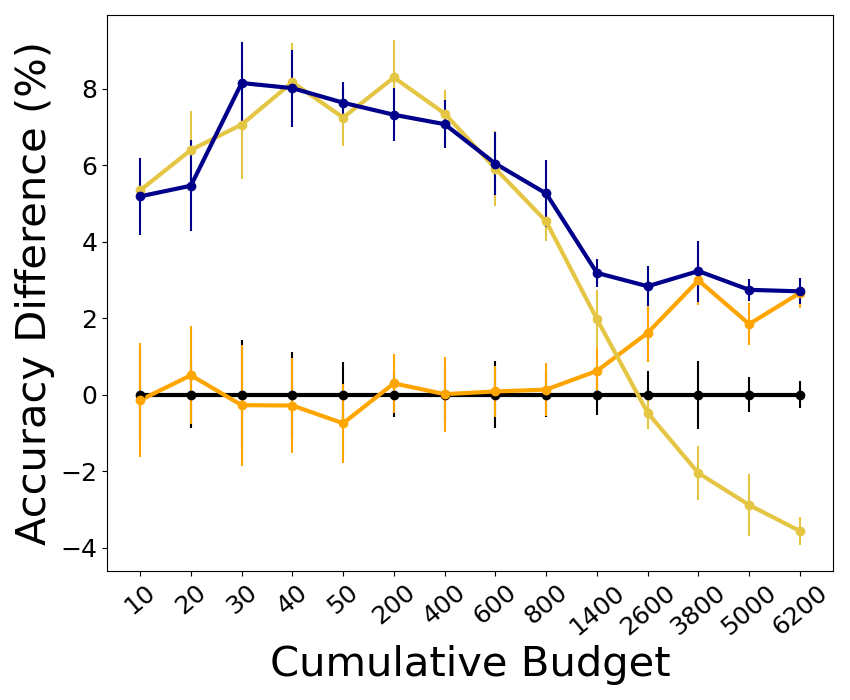}
    \caption{CIFAR-10}
    \label{fig:break_down_fig:a}
  \end{subfigure}\hfill
  \begin{subfigure}{0.49\linewidth}
    \centering
    \includegraphics[width=\linewidth]{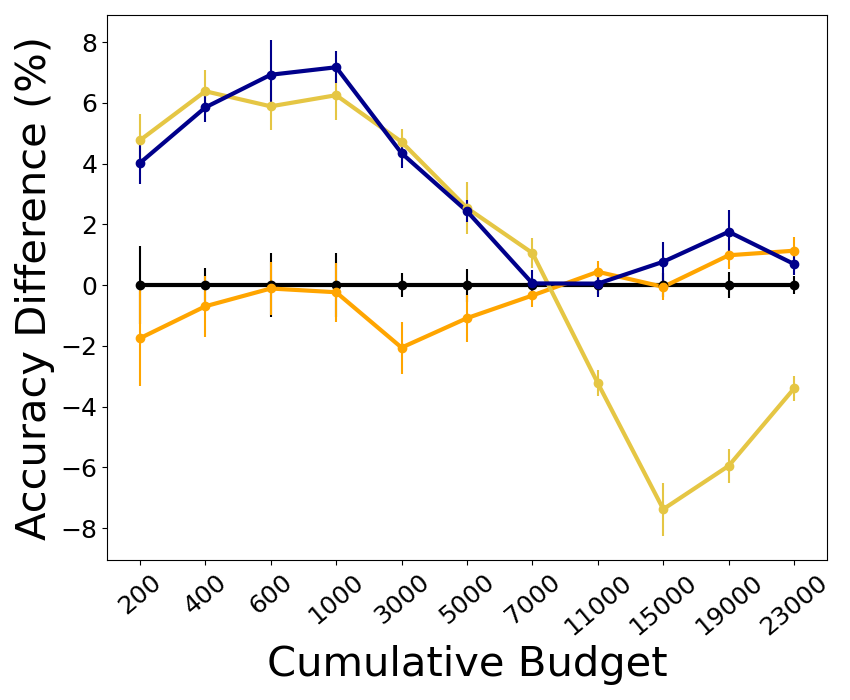}
    \caption{IMAGENET-50}
    \label{fig:break_down_fig:b}
  \end{subfigure}
  \caption{Performance (mean accuracy and STE) when optimizing each component of objective function (\ref{eq:obj_func_def}) separately and together.}
  \label{fig:break_down_fig}
\end{figure}

% \myparagraph{Empirical evaluation of $\boldsymbol{\cS}$}
% %
% Fig.~\ref{app:fig:comp_score_illust} tracks the behavior of the competence score $\cS$ across 3 datasets. In all cases, the comparison of Figs.\ref{app:fig:comp_score_illust:a},\ref{app:fig:comp_score_illust:b} reveals that $\cS$ increases monotonically with competence as measured by accuracy. Interestingly, while the critical point where $\cS$ changes from favoring $\mS$ (small values) to favoring $\mF$ (large values) corresponds to an intermediate level of accuracy, this transition is delayed for the more difficult datasets. 

% \begin{figure}[htb]
% \center
% \begin{subfigure}{.49\columnwidth}
% \includegraphics[width=\linewidth]{results/competence_score_illustration/random_acc.png}
% \caption{Accuracy with random sampling.}
% \label{fig:comp_score_illust:a}
% \end{subfigure}
% \begin{subfigure}{.49\columnwidth}
% \includegraphics[width=\linewidth]{results/competence_score_illustration/comp_score.png}
% \caption{$\cS$ during \MethodName\ training.}
% \label{fig:comp_score_illust:b}
% \end{subfigure}
% \caption{(a) Accuracy (mean and standard error) as a function of budget $b$, learning to classify CIFAR-100, ImageNet-50, and ImageNet-100 using a random sample of $b$ points. (b) The competence score $\cS$ as a function of budget. }

% \label{fig:comp_score_illust}
% \end{figure}

\myparagraph{Different uncertainty-based methods.}
% \label{sec:ablation:diff_uncertainty_methods}
%
%We now evaluate our choice for ${{\mF}(x)}$ in objective function (\ref{eq:obj_func_def}), modifying the specific method used to measure the uncertainty of the current learner. Specifically, we use the common methods of min margin, entropy of the min max activation, and uncertainty (see details in Section~\ref{sec:methodology}). Additionally, we use a score reminiscent of the AL method BADGE  \cite{gal2017deep} and termed \emph{Gradient norm}, which computes the norm of the gradient between the learner's last two layers. Once again, the results of \MethodName\ show robustness to this choice, with insignificant differences between the different uncertainty scores, see  Fig.~\ref{fig:uncertainty_methods_ab}. Consequently, we prioritize computational efficiency and opt for the fastest-to-compute score, which is min-margin. 
%
We evaluate our choice for ${{\mF}(x)}$ in objective function (\ref{eq:obj_func_def}), investigating the alternative methods of min margin, entropy of the min max activation, and uncertainty (see details in Section~\ref{sec:methodology}). Additionally, we use a score reminiscent of the AL method BADGE  \cite{gal2017deep} and termed \emph{Gradient norm}, which computes the norm of the gradient between the learner's last two layers. Once again, the results of \MethodName\ show robustness to this choice, with insignificant differences between the different uncertainty scores, see  Fig.~\ref{fig:uncertainty_methods_ab}. Consequently, we prioritize computational efficiency and opt for the fastest-to-compute score, which is min-margin.

\begin{figure}[htp]
\centering
% \begin{minipage}[t]{1\columnwidth}
  \centering
  \begin{subfigure}{0.49\linewidth}
    \centering
      \includegraphics[width=\linewidth]{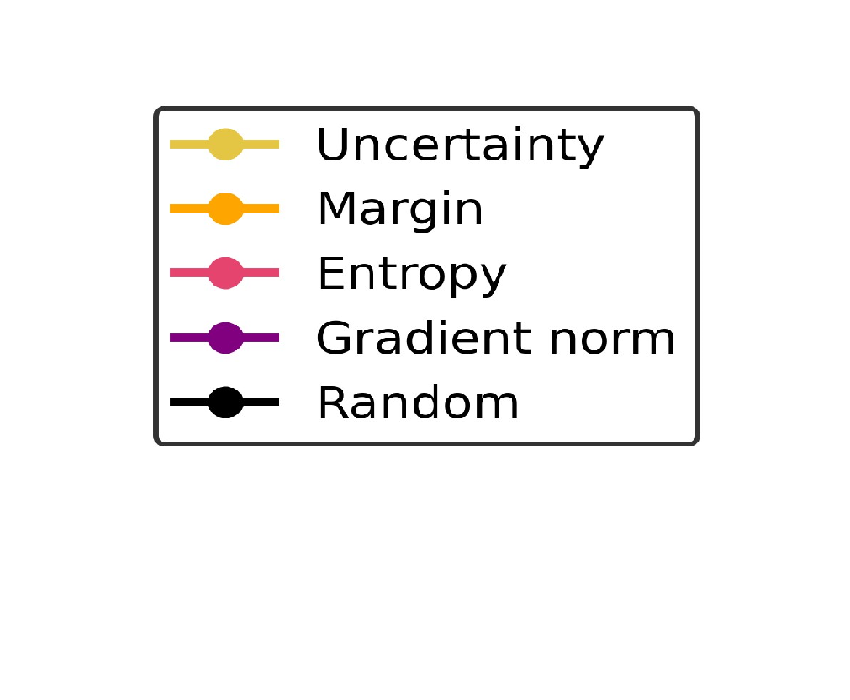}
  \end{subfigure}
  \begin{subfigure}{0.49\linewidth}
    \centering
      \includegraphics[width=\linewidth]{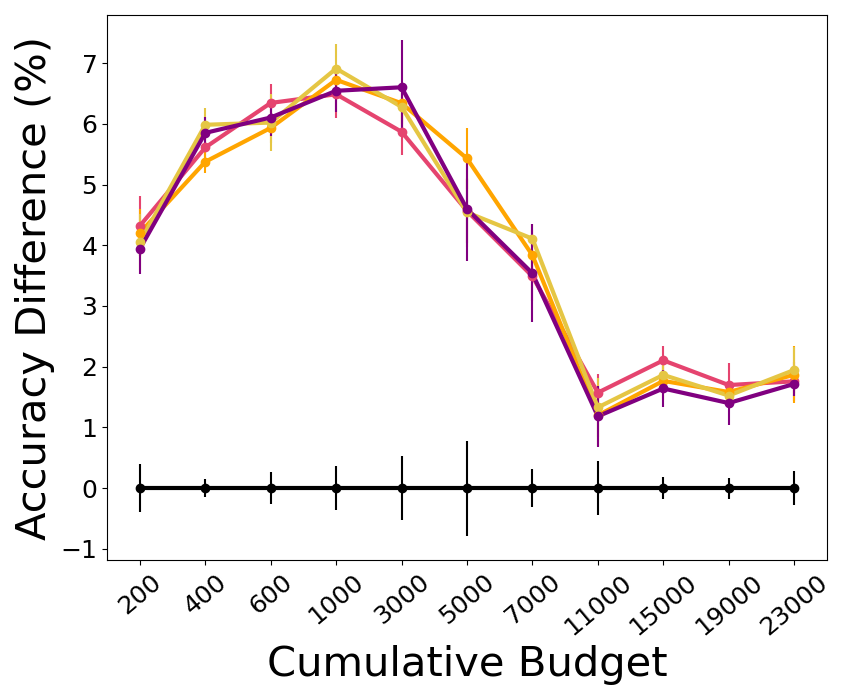}  
  \end{subfigure}
\caption{Similar to Fig.~\ref{fig:main_results},
% and Table~\ref{tab:imagenet100_results}
 using only  \MethodName\ while leveraging 4 different variants for its underlying uncertainty score (see text). %Each variation incorporates a dynamic algorithm weighted with a different uncertainty-based method. Performance is assessed by the accuracy difference between each variation and no active learning (random sampling). Each plot corresponds to $3-5$ repetitions using the CIFAR-100 dataset. These findings suggest that when the algorithm begins with optimization on ${{\mS}(x)}$, 
The performance differences between the variants are hardly significant.}
\label{fig:uncertainty_methods_ab}
\end{figure}

\section{Summary}
\label{sec:summary}
We investigate the challenge of Active Learning (AL) in a multi-budget setting. Our approach is motivated by a new bound on the generalization error of the Nearest Neighbor classifier, and aims to minimize this bound. Additionally, the sampling strategy is dynamically adjusted during training using a new competence score, also motivated by the same bound. We validate our method empirically in both supervised and semi-supervised frameworks across various datasets. The results show that \MethodName\ significantly outperforms alternative methods across all budgetary constraints.

% \myparagraph{Future directions.} In our future endeavors, we plan to explore:
% \begin{inparaenum}[(i)]
% \item potential avenues for enhancing the selection of $a$ and $k$ by leveraging existing labels or inferring a score for them through the topology of the initial self-supervised feature graph;
% \item soft-coverage methodologies, where the notion of coverage is probabilistic rather than binary, such as utilizing Gaussian measures. {@daphna - I need you to check that}
% \end{inparaenum}

\section*{Acknowledgments}
This work was supported by grants from the Israeli Council of Higher Education, AFOSR award FA8655-24-1-7006, and the Gatsby Charitable Foundation.
% \inbal{todo}

% \subsubsection{The Following Must Appear in Your Preamble}
% \begin{quote}
% \begin{scriptsize}\begin{verbatim}
% \documentclass[letterpaper]{article}
% % DO NOT CHANGE THIS
% \usepackage[submission]{aaai25} % DO NOT CHANGE THIS
% \usepackage{times} % DO NOT CHANGE THIS
% \usepackage{helvet} % DO NOT CHANGE THIS
% \usepackage{courier} % DO NOT CHANGE THIS
% \usepackage[hyphens]{url} % DO NOT CHANGE THIS
% \usepackage{graphicx} % DO NOT CHANGE THIS
% \urlstyle{rm} % DO NOT CHANGE THIS
% \def\UrlFont{\rm} % DO NOT CHANGE THIS
% \usepackage{graphicx}  % DO NOT CHANGE THIS
% \usepackage{natbib}  % DO NOT CHANGE THIS
% \usepackage{caption}  % DO NOT CHANGE THIS
% \frenchspacing % DO NOT CHANGE THIS
% \setlength{\pdfpagewidth}{8.5in} % DO NOT CHANGE THIS
% \setlength{\pdfpageheight}{11in} % DO NOT CHANGE THIS
% %
% % Keep the \pdfinfo as shown here. There's no need
% % for you to add the /Title and /Author tags.
% \pdfinfo{
% /TemplateVersion (2025.1)
% }
% \end{verbatim}\end{scriptsize}
% \end{quote}

%\clearpage
\bibliography{bib}

\appendix
\section*{Appendix}

\section{Implementation details}
\label{app:impl_details}

The source code for this study can be found in the zip file included in the supplementary material. This zip file contains a USAGE.md file that provides instructions on how to run the algorithm. The code in the zip file is based on the GitHub repositories of \citet{yehuda2022active} and \citet{munjal2022towards}. Our GitHub repository will be made public upon acceptance.

% \href{https://github.com/avihu111/TypiClust}{\texttt{https://github.com/avihu111/TypiClust}}
% \\

\MethodName\ necessitates the computation of the initial delta for each dataset representation. The values utilized in our experiments can be found in Table~\ref{tab:init_deltas_val}. The $\delta$ values of CIFAR-10 and CIFAR-100 are sourced from \citep{yehuda2022active}. Information on the computation of $\delta$ values for other datasets can be located in \app\ref{app:delta_initialization}.

\begin{table}[ht]
  \centering
  \scriptsize
  \begin{tabular}{llll}
    \toprule    
    Dataset&  SSL method&  $\delta_0$ \\ 
    \midrule
    STL-10&  SimCLR&  0.55\\  
    SVHN&  SimCLR&  0.4\\  
    CIFAR-10&  SimCLR&  0.75\\  
    CIFAR-100&  SimCLR&  0.65\\  
    ImageNet-50&  DINOv2&  0.75\\  
    ImageNet-100&  DINOv2&  0.7\\  
    ImageNet-200&  DINOv2&  0.65\\ 
    CIFAR-100&  MOCOv2+&  0.5\\  
    CIFAR-100&  BYOL&  0.5\\  
    CIFAR-100&  Barlow Twins&  0.55\\  
    \bottomrule
  \end{tabular}
  \caption{Initial delta values used in experiments}
  \label{tab:init_deltas_val}
\end{table}

Additionally, throughout our experiments we used identical algorithm parameters, distinguishing between those for datasets with smaller class amounts (10) and larger ones (50+). These parameters encompass the adaptive purity threshold $\tau$ from Alg~\ref{alg:Delta_Adjustment} and the logistic function parameters for the competence score $\cS(\sL, \Delta)$ in (\ref{def:competence_score}), $(a, k)$. Specifically, we employed $\tau = 0.2~\text{cover} + 0.4$, with logistic parameters $a = 0.9$ and $k = 30$ for CIFAR-10, SVHN, and STL-10, and $a = 0.8$ and $k = 30$ for CIFAR-100 and ImageNet subsets. The $\delta$ resolution for the binary search in \MethodName\ was set to $0.05$. Our ablation study, detailed in \app\ref{app:hyper_parameters_exp}, demonstrates that these selections have negligible impact on performance.

\subsection{Supervised training}
\label{app:supervised_hyper_params}
When training on STL-10, SVHN, CIFAR-10 and CIFAR-100 datasets, we utilized a ResNet-18 architecture trained for 200 epochs. Our optimization strategy involved using an SGD optimizer with a Nesterov momentum of 0.9, weight decay set to 0.0003, and cosine learning rate scheduling starting at a base rate of 0.025. Training was performed with a batch size of 100 examples, and we applied random cropping and horizontal flips for data augmentation. For an illustration of these parameters in use, refer to \citep{munjal2022towards}.

When training ImageNet-50, we used the same hyper-parameters, only changing the base learning rate to $0.01$, the batch size to $50$ and the epoch amount to $50$.

When training ImageNet-100/200, we used the same hyper-parameters as ImageNet-50, only changing the epoch amount to $100$.

The train and test partitions followed the original sets from the corresponding paper, with $10\%$ of the train set used for validation in all datasets except for ImageNet-200, where $5\%$ was used for validation.

\subsection{Semi-supervised training}
\label{app:semi_supervised_hyper_params}
When training FlexMatch \citep{zhang2021flexmatch}, we used the semi supervised framework by \citep{usb2022}, and the AL framework by \citep{munjal2022towards}. All experiments involved 3 repetitions.

We relied on the standard hyper-parameters used by FlexMatch \citep{zhang2021flexmatch}.
Specifically, we trained WideResNet-28 for 512 epochs using the SGD optimizer, with $0.03$ learning rate, $64$ batch size, $0.9$ SGD momentum, $0.999$ EMA momentum, $0.001$ weight decay and $2$ widen factor. 

\subsection{Self-supervised feature extraction}
\label{app:SSL_feature_extraction}
\myparagraph{STL-10, SVHN, CIFAR-10, CIFAR-100.}
To extract semantically meaningful features, we trained SimCLR using the code provided by \citep{van2020scan} for STL-10, SVHN, CIFAR-10 and CIFAR-100. Specifically, we used ResNet-18 \citep{DBLP:conf/cvpr/HeZRS16} with an MLP projection layer to a $128$-dim vector, trained for 512 epochs. All the training hyper-parameters were identical to those used by SCAN (all details can be found in \citep{van2020scan}). After training, we used the $512$ dimensional penultimate layer as the representation space.

\myparagraph{Representation learning: ImageNet-50/100/200.}
We extracted features from the official (ViT-S/14) DINOv2 weights pre-trained on ImageNet \citep{oquab2023dinov2}. We employed the L2-normalized penultimate layer for the embedding, which has a dimensionality of $384$.

\myparagraph{Ablation embeddings -- CIFAR-100.}
We extracted more features as BYOL \citep{grill2020bootstrap}, MOCOv2+ \citep{he2020momentum}, and Barlow Twins \citep{zbontar2021barlow} using pre-trained weights from \citep{JMLR:v23:21-1155}. 

\section{Purity Calculation in nNN}
\label{app:nNN_explanation}

When transitioning from 1NN to nNN, the structure of the problem remains unchanged, with the key difference being that the radius $\delta$ is adjusted for each labeled example. Next, we demonstrate that the purity of each ball in the nNN problem can be precisely computed using these $\delta$-values. This is because when a labeled example \(c \in \sL\) assigns a label to a point $x$, it must also cover $x$. Therefore, any mislabeling by $c$ would reduce the purity, as illustrated by the following statement:

\begin{proposition}
If \(c \in \sL\) is labeling \(x\) and \(x \in C(\sL, \Delta)\), then \(x \in B_{\delta_c}(c)\).
\end{proposition}

\begin{proof}
Let \(c \in \sL\) be the nearest neighbor (nNN) of \(x\), and let \(\delta_c\) be the corresponding radius for \(c\). By definition, \(f^{nNN}(x) = f(x)\), and since \(c\) is labeled, we have \(f(c) = f^{nNN}(c)\equiv f^{N}(c)\). 

Given that \(x \in C(\sL, \Delta)\), there exists a labeled example \(c' \in \sL\) that covers \(x\), meaning \(d(x, c') < \delta_{c'}\). Suppose that \(c \neq c'\). Since \(c\) is the nNN of \(x\), it follows that:
\[
\frac{d(x, c)}{\delta_c} < \frac{d(x, c')}{\delta_{c'}}
\]
Given \(d(x, c') < \delta_{c'}\), we get:
\[
\frac{d(x, c)}{\delta_c} < 1 \Rightarrow d(x, c) < \delta_c
\]
Thus, \(x \in B_{\delta_c}(c)\), meaning \(c\) covers \(x\) within the radius \(\delta_c\).

This implies that if \(x\) were wrongly labeled, \(c\) would still cover \(x\), indicating lower purity in \(B_{\delta_c}(c)\).
\end{proof}

\section{Hyper-parameters exploring}
\label{app:hyper_parameters_exp}
\myparagraph{Theoretical basis for the Choice of $\boldsymbol{a}$ and $\boldsymbol{k}$.} The parameter $a$ represents the center of the sigmoid function. Our objective is for the algorithm to assign increased weight to $\mF$ only after it has covered a substantial portion of the dataset. Consequently, we choose $a \in (0.5, 1)$. The parameter $k$ controls the steepness of the sigmoid curve. The function must yield 0 in the absence of coverage and approach 1 as coverage becomes complete. To achieve this, we configure the numerator of the fraction to be equal to the denominator when the cover's probability is 1. Although the function cannot pass precisely through $(0,0)$, it can asymptotically approach $(0, \varepsilon)$ for an infinitesimally small $\varepsilon \in \Re$. 

We now investigate the relationship between this approximation error and the parameter $k$. Define the function \( f(x) \) as $\cS$ for simplicity:
\[
f(x) = \frac{1 + e^{-k(1-a)}}{1 + e^{-k(x-a)}}
\]
The limit of $f(x)$ as \( x \) approaches 0 is
\[
\lim_{x \to 0} f(x) = \frac{1 + e^{-k(1-a)}}{1 + e^{ka}}
\]
For the function to approach 0, we require:
\[
\frac{1 + e^{-k(1-a)}}{1 + e^{ka}} \ll 1
\]
This implies:
\[
1 + e^{-k(1-a)} \ll 1 + e^{ka}
\]
\[
e^{-k(1-a)} \ll e^{ka}
\]
\[
e^{-k} \ll e^{ka}
\]
\[
e^{-k} \ll 1 \tag{*}
\]

Next, we calculate \( f(0) \):
\[
f(0) = \frac{1 + e^{-k(1-a)}}{1 + e^{ka}} = \varepsilon
\]
Solving for \( k \):
\[
1 + e^{-k(1-a)} = \varepsilon (1 + e^{ka})
\]
\[
1 - \varepsilon = \varepsilon e^{ka} - e^{-k(1-a)}
\]
\[
1 - \varepsilon = e^{ka} (\varepsilon - e^{-k})
\]
\[
e^{ka} = \frac{1 - \varepsilon}{\varepsilon - e^{-k}}
\]

Approximating $e^{ka}$ using (*):
\[
e^{ka} \approx \frac{1 - \varepsilon}{\varepsilon}
\]
So:
\[
ka \approx \ln \left(\frac{1 - \varepsilon}{\varepsilon}\right)
\]
Finally:
\[
k \approx \frac{1}{a} \ln \left(\frac{1}{\varepsilon} - 1\right) \approx \frac{\ln \left(\frac{1}{\varepsilon}\right)}{a}
\]

To minimize the approximation error, we should choose a larger \( k \). In our experiments with \( a = 0.8 \text{ or } 0.9 \) and \( \varepsilon = 10^{-10} \), we determined:
\[
k(a=0.9, \varepsilon=10^{-10}) \approx 25.58
\]
\[
k(a=0.8, \varepsilon=10^{-10}) \approx 28.78
\]
The value of $\varepsilon$ was chosen to be asymptotically smaller than the dimension of the initial graphs constructed for all the datasets. In accordance, we chose \( k = 30 \).

\myparagraph{Experimental assessment of \MethodName\ hyper-parameter consistency}
As described before, \MethodName\ has 3 hyper-parameters: $a, k \text{ and } \tau$. When running the main experiments described in Section~\ref{sec:main_results}, we maintained consistency by employing identical parameters for STL-10, SVHN, and CIFAR-10, as well as for CIFAR-100 and the subsets of ImageNet. This approach helps to demonstrate that the choice of the hyper-parameters is not critical and does not significantly affect the results. Here, we conduct experiments with \MethodName\ using STL-10 and CIFAR-100 datasets and several values for $a$ and $k$ from the definition of the competence score $\cS(\sL, \Delta)$ in (\ref{def:competence_score}). We observe that while certain choices may be slightly better than others, all choices yield similar results, and \MethodName\ consistently outperforms or matches the performance of previous methods.

\begin{figure}[htb]  
\center
\begin{subfigure}{1\columnwidth}
% legends
\includegraphics[width=\linewidth]{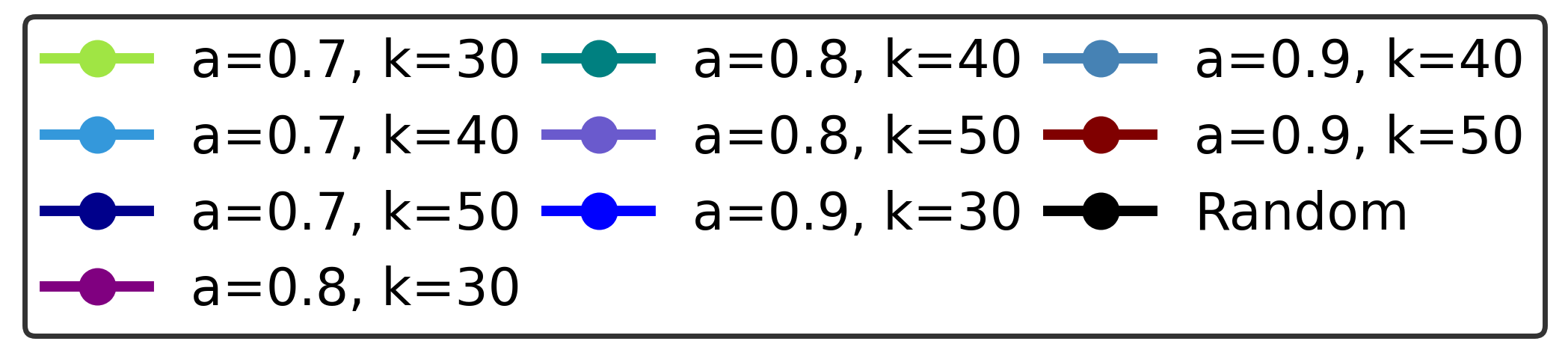}
\label{app:fig:params_exp:a}
\end{subfigure}
\begin{subfigure}{.49\columnwidth}
\includegraphics[width=\linewidth]{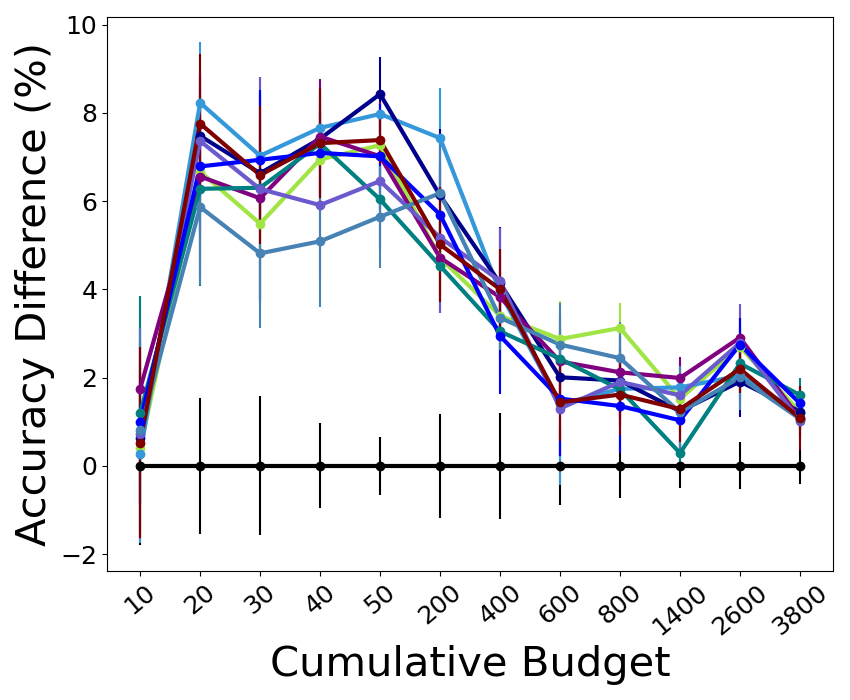}
\caption{STL-10} 
\label{app:fig:params_exp:b}
\end{subfigure}
\begin{subfigure}{.49\columnwidth}
\includegraphics[width=\linewidth]{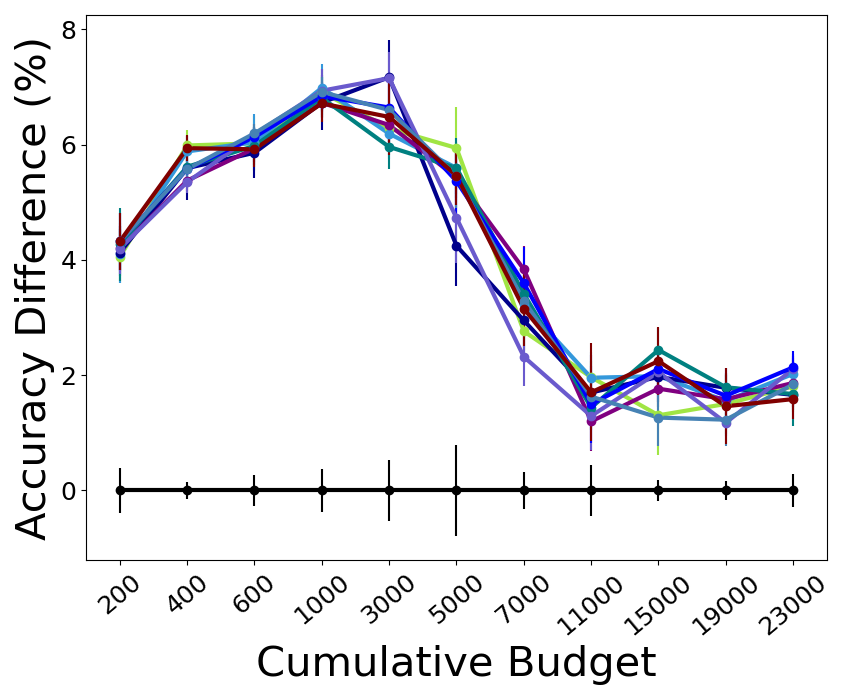}
\caption{CIFAR-100} 
\label{app:fig:params_exp:c}
\end{subfigure} 
\caption{Evaluation of \MethodName\ across different settings of the logistic function hyper-parameters a and k. Each plot displays the mean and standard error over $3-5$ repetitions. The findings indicate that fine-tuning these parameters has negligible effects on performance.}
\label{app:fig:params_exp}
\vspace{-.5cm}
\end{figure}

\section{Time and space complexity of \MethodName}
\label{app:complexity_analysis}

% During neural network training, \MethodName\ is invoked once to determine the optimal subset for human annotation, which is then utilized for network training.

% Prior to active sampling, a preliminary step involves selecting the initial $\delta$ for the first adjacency graph. This procedure is detailed in \citep{yehuda2022active}.

% Following this initial step and preceding the active sampling, the algorithm proceeds with customizing appropriate $\delta$ values for each labeled example from the previous iteration.

% After this step, the active sampling commences and encompasses several subsequent stages. The comprehensive process is elaborated upon in Section~\ref{sec:method}.

For the complexity analysis below, let $n$ represent the number of examples in the combined unlabeled and labeled pool $|\sU\cup\sL|$, $d$ the dimension of the data embedding space, $b$ - the given budget ($|\sL|$) and $q$ - the size of the query set. As previously stated, \MethodName\ can be divided into three distinct steps:

\subsection{Selection of initial $\boldsymbol{\delta}$ delta}
\label{sub:initial_delta_comp}
\myparagraph{Time Complexity}: The process begins by creating $t$ adjacency graphs, each corresponding to a different $\delta$ value being evaluated. Here, $t$ denotes the number of $\delta$'s being examined.
The time complexity of each graph generation is $O(n^2 \cdot d)$ time. The computation of purity requires the prediction of pseudo-labels by k-means at a cost of $O(n^2)$. The complexity of computing ball purity for one example is $O(n^2)$, so totally the ball purity computing for one graph takes $O(n^3)$. This adds up to overall time complexity of $O(tdn^2 + n^2 + tn^3) = O(tn^3)$. We can construct the initial graph using the largest \(\delta\)-value and then, instead of building a new graph for each \(\delta\)-value in \(O(tdn^2)\), we can filter distances above the current \(\delta\) in \(O(dn^2 + n^2)\). However, the total complexity remains \(O(tn^3)\). Although it may seem significant, this step doesn't occur during the active learning algorithm process. Additionally, it's possible to confine the data employed for purity computation to the densest points, effectively lowering the complexity to $O(dn^2)$. This approach mandates the use of at least the number of points equivalent to the number of classes for purity computation. This technique yields similar $\delta_0$ values (consistent for CIFAR-10, CIFAR-100, and ImageNet-50, and with a variation of 0.05 for ImageNet-100/200 and STL-10 and 0.15 for SVHN).

\myparagraph{Space complexity}: Naively, the space complexity is $O(n^2)$, which might be impractical for large datasets like ImageNet. However, \citep{yehuda2022active} demonstrates that utilizing a sparse matrix in coordinate list (COO) format and using limited $\delta$ values, the space complexity becomes $O(|E|)$, where $E$ represents the set of edges in the graph. Although $O(|E|)$ remains $O(n^2)$ in the worst-case scenario, in practice, the average vertex degree with a limited radius is less than $n$.

\subsection{$\boldsymbol{\Delta}$-Adjustment}
\myparagraph{Time complexity}:
Each iteration of \MethodName\ starts with customizing $\delta$ values for each labeled example from the last active step. In the worst case, there are $|\sL|$ examples. Customizing $\Delta$ involves building the adjacency graph using $\delta_{\text{max}}$, running all examples on the model to obtain their predictions, and applying binary search over $\delta$ values between 0 and $\delta_{\text{max}}$ for each example from the last active step. The binary search is performed on a continuous parameter ($\delta$) with a search resolution $r$. This implies that there are $\frac{\delta_{\text{max}}}{r}$ available values for $\delta$. For each $\delta$ value, the algorithm computes its purity value.
Overall, this process takes $O(dn^2 + n + |\sL| \cdot \log(\frac{\delta_{\text{max}}}{r})\cdot n^2) = O(dn^2 + b \cdot \log(\frac{\delta_{\text{max}}}{r})\cdot n^2)$. In practical terms, due to the vectorization of these processes, it requires approximately 27 minutes to construct an adjacency graph for ImageNet-200 and update deltas for 1000 examples on a single CPU.

\myparagraph{Space complexity}: Similar to the space complexity analysis in \app\ref{sub:initial_delta_comp}, which is $O(n^2)$.

\subsection{Active sampling}
\myparagraph{Time complexity}: Active sampling involves balancing model min-margin and example density based on the current adjacency graph. Initially, the margin of each example is computed using the softmax output from the model's last layer, requiring $O(n)$ time. Subsequently, the current adjacency is created using $\delta_{\text{avg}}$, which is the average over the $\sL$ $\delta$ values list - $\Delta$. As discussed in \app\ref{sub:initial_delta_comp}, this step has a time complexity of $O(dn^2)$.
Following this, the iterative process for selecting a single sample includes the following steps:
\begin{itemize}[noitemsep]
    \item Calculating node degrees -- $O(|E|)$ time.
    \item Finding the node with maximal degree -- $O(n)$ time.
    \item Removing incoming edges from the graph for covered points -- $O(|E|)$ time.
\end{itemize}

Samples are iteratively selected from the current sparse graph, with incoming edges to newly covered samples being removed. Unlike adjacency graph creation, sample selection cannot be parallelized, as each step depends on the previous one.
Overall, the time complexity for selecting a sample is $O(|E|+n)$, resulting in a worst-case overall complexity of $O(dn^2)$. However, as more points are selected, the removal of edges speeds up the selection of later samples. Practically, it consumes about 15 minutes to construct an adjacency graph and select 1000 samples from ImageNet-200 on a single CPU.

\myparagraph{Space complexity}: Similar to the previous space complexity analysis -- $O(n^2)$.

\section{Additional empirical results}
\label{app:add_emp_results}

\subsection{$\boldsymbol{\delta_0}$ Initialization}
\label{app:delta_initialization}
To set the initial value $\delta_0$, we adopted the method outlined in \citep{yehuda2022active}, as detailed in Section~\ref{sec:method}. Fig.~\ref{app:fig:init_delta} displays the purity function across various $\delta$ values, along with the selected $\delta$ values for each dataset. As mentioned previously, the $\delta$ values for CIFAR-10 and CIFAR-100 are derived from \citep{yehuda2022active}.

\begin{figure}[htb] % htb
\center
\begin{subfigure}{.49\columnwidth}
\includegraphics[width=\linewidth]{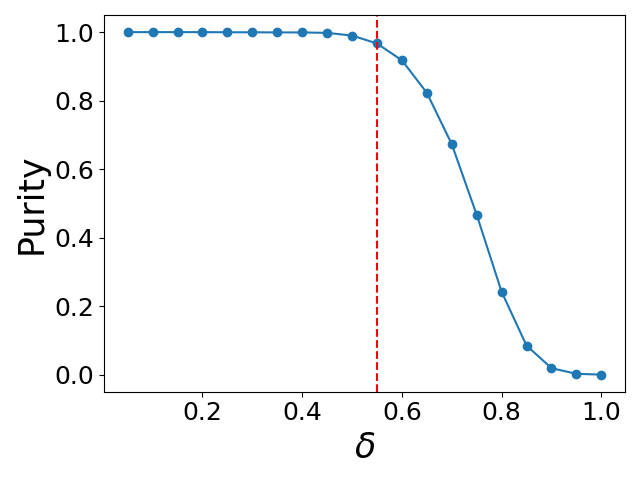}
\caption{STL-10} 
\label{app:fig:init_delta:a}
\end{subfigure}
\begin{subfigure}{.49\columnwidth}
\includegraphics[width=\linewidth]{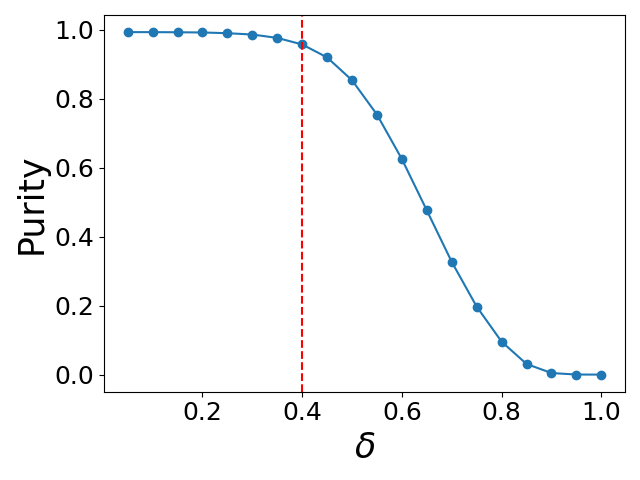}
\caption{SVHN} 
\label{app:fig:init_delta:b}
\end{subfigure}
\begin{subfigure}{.49\columnwidth}
\includegraphics[width=\linewidth]{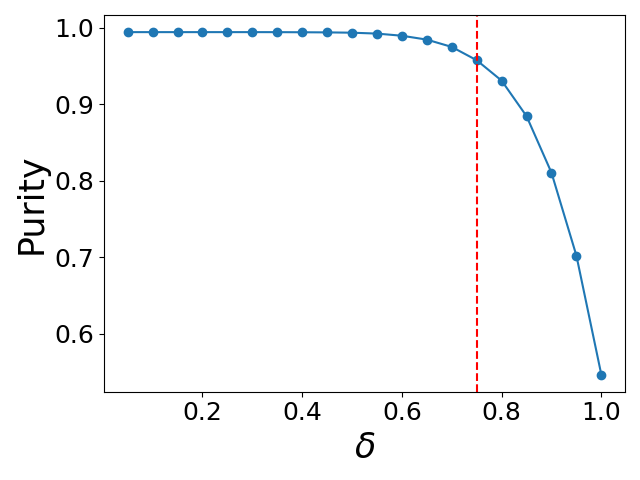}
\caption{ImageNet-50} 
\label{app:fig:init_delta:c}
\end{subfigure}
\begin{subfigure}{.49\columnwidth}
\includegraphics[width=\linewidth]{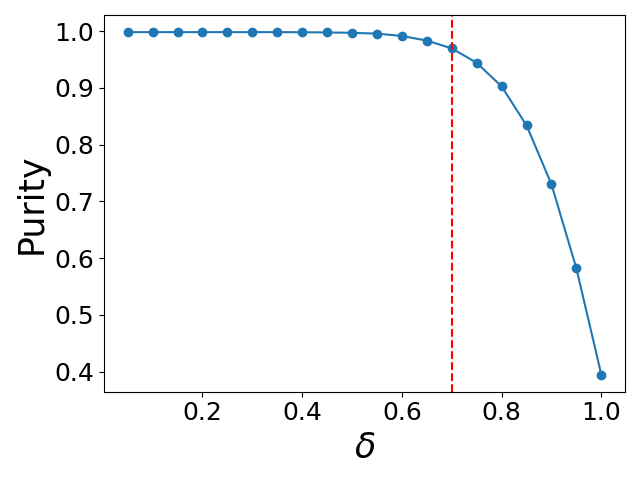}
\caption{ImageNet-100}
\label{app:fig:init_delta:d}
\end{subfigure}
% \begin{subfigure}{.325\textwidth}
% \includegraphics[width=\linewidth]{results/deltas_choosing/IMAGENET200.png}
% \caption{ImageNet-200} 
% \label{app:fig:init_delta:e}
% \end{subfigure}
\caption{$\pi(\delta)$, estimated from the unlabeled data and using k-means algorithm for labeling. The dashed line indicates the highest $\delta$, after which purity drops below $\alpha=0.95$.
}
\label{app:fig:init_delta}
\end{figure}

\subsection{$\boldsymbol{\Delta}$ distribution evolution throughout training}

\label{app:Delta_distribution}
The experiment in Fig.~\ref{app:fig:Del_dist} illustrates the distribution of $\Delta$ throughout the active selection and training phases of \MethodName. The significant standard deviation of $\Delta$ compared to the minor standard error underscores the significance of employing a dynamic method that assigns a distinct radius to each ball.
\begin{figure}[htb]  
\center
\begin{subfigure}{.49\columnwidth}
\includegraphics[width=\linewidth]{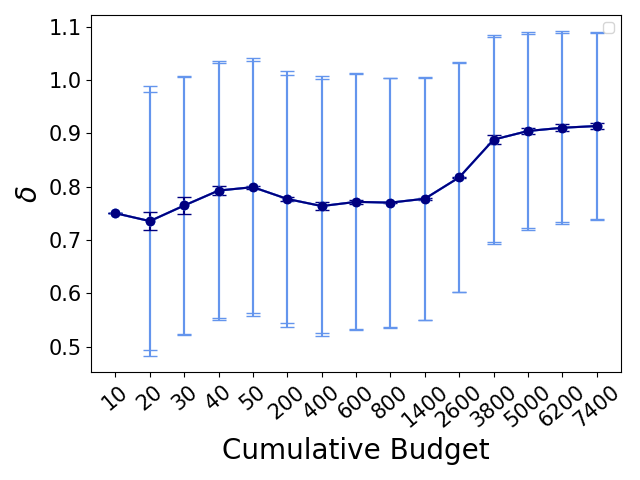}
\caption{CIFAR-10} 
\label{app:fig:Del_dist:a}
\end{subfigure}
\begin{subfigure}{.49\columnwidth}
\includegraphics[width=\linewidth]{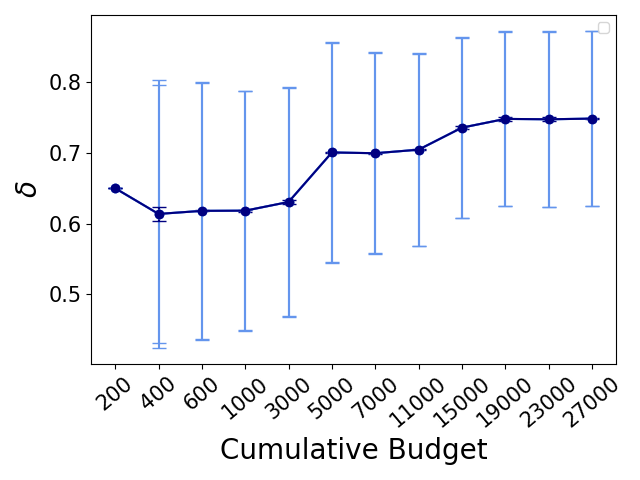}
\caption{CIFAR-100} 
\label{app:fig:Del_dist:b}
\end{subfigure}
% \begin{subfigure}{.325\textwidth}
% \includegraphics[width=\linewidth]{results/delta_distribution/DCMM/imagenet50_8030-8032.png}
% \caption{ImageNet-50} 
% \label{app:fig:Del_dist:c}
% \end{subfigure}
\caption{The $\Delta$ distribution through \MethodName\ algorithm.  Each plot displays the distribution over 5-10 repetitions. The results indicate that as the emphasis shifts towards maximizing the purity of sample balls, rather than solely focusing on maximum coverage, the radii adapt accordingly. This, combined with the improved performance observed with \MethodName, highlights the necessity of the dynamic algorithm in enhancing the effectiveness of active sampling.}
\label{app:fig:Del_dist}
\end{figure}

\begin{table*}[t!]
\centering
\scriptsize
\setlength{\tabcolsep}{1pt} % Adjust horizontal space between cells
\begin{tabularx}{0.995\textwidth}{l*{11}{>{\centering\arraybackslash}X} r}
\toprule
$|\sL|$ & Random & Prob Cover & BADGE & BALD & Coreset & Uncertainty & Entropy & DBAL & Margin & LDM-s & DCoM \\ 
\midrule
10 & 14.53±0.60 & \textbf{20.39±0.23} & 15.17±0.94 & 10.81±0.61 & 14.65±0.74 & 12.99±0.77 & 12.78±0.72 & 11.74±0.69 & 14.39±0.88 & - & 19.72±0.40 \\ 
20 & 16.82±0.44 & \textbf{24.64±0.73} & 17.85±0.77 & 13.79±0.89 & 16.09±0.74 & 16.38±0.52 & 17.07±0.62 & 17.00±0.49 & 17.33±0.85 & - & 22.28±0.75 \\ 
30 & 18.82±0.71 & \textbf{27.19±0.25} & 20.38±0.67 & 16.25±0.83 & 18.22±0.93 & 18.00±0.68 & 19.26±0.44 & 18.34±0.50 & 18.55±0.87 & - & \textbf{26.98±0.35} \\ 
40 & 20.89±0.55 & \textbf{29.46±0.38} & 22.19±0.45 & 17.93±0.67 & 19.79±0.61 & 19.00±0.70 & 19.98±0.60 & 20.30±0.35 & 20.61±0.68 & - & \textbf{28.91±0.46} \\ 
50 & 22.75±0.42 & 29.80±0.21 & 22.41±0.35 & 18.36±0.54 & 21.26±0.63 & 21.63±0.70 & 22.17±0.48 & 21.06±0.45 & 22.00±0.61 & - & \textbf{30.39±0.12} \\ 
200 & 31.53±0.29 & \textbf{38.83±0.18} & 32.10±0.42 & 27.95±0.29 & 30.91±0.76 & 30.69±0.27 & 29.51±0.42 & 28.91±0.50 & 31.83±0.48 & - & \textbf{38.86±0.40} \\ 
400 & 38.64±0.36 & \textbf{45.36±0.51} & 37.92±0.37 & 35.10±0.33 & 37.85±0.70 & 36.48±0.46 & 36.71±0.38 & 36.01±0.32 & 38.65±0.62 & - & \textbf{45.72±0.27} \\ 
600 & 43.66±0.44 & \textbf{49.13±0.36} & 42.69±0.35 & 40.75±0.29 & 42.48±0.39 & 41.96±0.24 & 40.75±0.23 & 41.19±0.53 & 43.75±0.22 & - & \textbf{49.71±0.39} \\ 
800 & 46.84±0.29 & \textbf{51.49±0.37} & 46.34±0.32 & 44.12±0.49 & 46.21±0.25 & 45.50±0.17 & 43.89±0.31 & 44.93±0.42 & 46.97±0.40 & - & \textbf{52.10±0.58} \\ 
1,400 & 54.35±0.27 & 56.33±0.26 & 54.06±0.42 & 53.08±0.34 & 52.31±0.31 & 53.13±0.28 & 52.93±0.34 & 53.28±0.50 & 54.97±0.35 & 50.30±0.00 & \textbf{57.54±0.10} \\ 
2,600 & 62.93±0.31 & 62.16±0.18 & 64.67±0.31 & 64.19±0.24 & 61.94±0.40 & \textbf{64.98±0.59} & 64.62±0.46 & 63.53±0.44 & 64.54±0.45 & 57.01±0.01 & \textbf{65.76±0.21} \\ 
3,800 & 68.61±0.44 & 66.83±0.25 & \textbf{71.36±0.32} & 70.77±0.53 & 68.41±0.34 & \textbf{71.90±0.36} & \textbf{71.30±0.36} & 70.55±0.25 & \textbf{71.60±0.21} & 61.30±0.01 & \textbf{71.84±0.36} \\ 
5,000 & 73.61±0.23 & 70.95±0.23 & 75.93±0.21 & 75.86±0.27 & 72.92±0.36 & 75.82±0.31 & \textbf{76.13±0.31} & 75.24±0.47 & 75.46±0.31 & 64.09±0.00 & \textbf{76.36±0.05} \\ 
6,200 & 75.95±0.18 & 73.51±0.30 & \textbf{78.90±0.27} & 77.92±0.36 & 77.00±0.19 & \textbf{78.84±0.27} & 78.37±0.20 & \textbf{78.54±0.26} & \textbf{78.62±0.22} & 65.79±0.00 & \textbf{78.66±0.16} \\ 
7,400 & 78.49±0.14 & 75.71±0.35 & \textbf{80.88±0.22} & \textbf{80.94±0.16} & 79.39±0.30 & \textbf{81.19±0.15} & \textbf{81.02±0.19} & \textbf{81.10±0.24} & \textbf{80.76±0.29} & 67.28±0.01 & \textbf{81.30±0.34} \\ 
\bottomrule
\end{tabularx}
\begin{center} {CIFAR-10 dataset} \end{center}
% \end{table}
% % Cifar-100
% \begin{table}[htbp]
% \caption{Model accuracy for different sizes of labeled set $\sL$ and AL strategies}% on CIFAR-100 dataset}
% \label{tab:res4}
% \centering
% \tiny
\setlength{\tabcolsep}{1pt} % Adjust horizontal space between cells
\begin{tabularx}{0.995\textwidth}{l*{11}{>{\centering\arraybackslash}X} r}
\toprule
$|\sL|$ & Random & Prob Cover & BADGE & BALD & Coreset & Uncertainty & Entropy & DBAL & Margin & LDM-s & DCoM \\ 
\midrule
200 & 6.39±0.19 & \textbf{10.79±0.08} & 6.31±0.09 & 3.30±0.11 & 6.39±0.19 & 2.59±0.05 & 2.74±0.17 & 2.61±0.18 & 5.09±0.37 & - & \textbf{10.59±0.19} \\ 
400 & 8.74±0.07 & 13.53±0.22 & 8.27±0.14 & 6.63±0.23 & 8.62±0.17 & 4.59±0.23 & 4.68±0.23 & 4.41±0.30 & 7.58±0.17 & - & \textbf{14.11±0.11} \\ 
600 & 10.73±0.13 & 15.86±0.16 & 10.29±0.27 & 8.33±0.31 & 10.39±0.25 & 6.52±0.38 & 6.32±0.27 & 5.94±0.28 & 10.20±0.16 & - & \textbf{16.66±0.11} \\ 
1,000 & 13.49±0.18 & 19.49±0.20 & 13.66±0.26 & 11.42±0.26 & 12.94±0.21 & 9.98±0.38 & 9.31±0.22 & 8.47±0.35 & 13.28±0.19 & - & \textbf{20.21±0.07} \\ 
3,000 & 24.02±0.27 & 29.31±0.42 & 24.24±0.28 & 22.10±0.11 & 23.00±0.28 & 21.93±0.26 & 20.15±0.33 & 20.29±0.23 & 23.71±0.31 & - & \textbf{30.36±0.08} \\ 
5,000 & 31.52±0.39 & 34.64±0.15 & 32.18±0.33 & 29.58±0.38 & 31.41±0.26 & 30.60±0.40 & 29.15±0.46 & 29.60±0.28 & 31.66±0.23 & - & \textbf{36.95±0.10} \\ 
7,000 & 37.69±0.16 & 38.36±0.13 & 38.09±0.10 & 34.36±0.20 & 38.24±0.16 & 35.87±0.19 & 35.11±0.11 & 35.67±0.35 & 38.16±0.17 & 31.85±0.00 & \textbf{41.52±0.25} \\ 
11,000 & 45.66±0.22 & 44.69±0.20 & \textbf{47.00±0.16} & 44.37±0.28 & 46.53±0.18 & 45.23±0.20 & 45.01±0.17 & 44.31±0.24 & 46.37±0.36 & 40.88±0.01 & \textbf{46.86±0.28} \\ 
15,000 & 50.57±0.09 & 49.58±0.46 & \textbf{52.42±0.17} & 51.02±0.09 & \textbf{52.19±0.29} & 51.82±0.16 & 51.43±0.21 & 51.30±0.35 & \textbf{52.37±0.19} & 46.59±0.01 & \textbf{52.33±0.22} \\ 
19,000 & 54.95±0.09 & 52.90±0.14 & 56.32±0.07 & 55.55±0.27 & \textbf{56.72±0.31} & \textbf{56.55±0.21} & 55.76±0.31 & 55.75±0.14 & \textbf{56.73±0.24} & 49.41±0.01 & \textbf{56.53±0.07} \\ 
23,000 & 57.74±0.14 & 56.29±0.21 & \textbf{59.60±0.16} & \textbf{59.96±0.35} & \textbf{59.59±0.22} & \textbf{59.79±0.23} & \textbf{59.91±0.08} & \textbf{59.72±0.21} & \textbf{59.88±0.21} & 52.48±0.00 & \textbf{59.60±0.33} \\ 
\bottomrule
\end{tabularx}
\begin{center} {CIFAR-100 dataset} \end{center}
% \end{table}
% % ImageNet-50
% \begin{table}[htbp]
% \caption{Model accuracy for different sizes of labeled set $\sL$ and AL strategies}% on ImageNet-50 dataset}
% \label{tab:res5}
% \centering
% \tiny
\setlength{\tabcolsep}{1pt} % Adjust horizontal space between cells
\begin{tabularx}{0.995\textwidth}{l*{11}{>{\centering\arraybackslash}X} r}
\toprule
$|\sL|$ & Random & Prob Cover & BADGE & BALD & Coreset & Uncertainty & Entropy & DBAL & Margin & DCoM \\ 
\midrule
200 & 8.25±0.65 & \textbf{12.33±0.34} & 8.82±0.18 & 4.46±0.31 & 8.31±0.52 & 4.26±0.26 & 4.11±0.09 & 4.26±0.24 & 6.50±0.94 & \textbf{12.27±0.04} \\ 
400 & 12.18±0.29 & 16.64±0.29 & 12.73±0.42 & 9.49±0.84 & 9.40±0.59 & 7.74±0.41 & 7.01±0.37 & 6.45±0.33 & 11.48±0.72 & \textbf{18.03±0.21} \\ 
600 & 14.85±0.53 & 20.38±0.47 & 15.69±0.38 & 11.21±0.41 & 10.68±0.35 & 10.38±0.27 & 11.28±0.54 & 10.34±0.46 & 14.74±0.35 & \textbf{21.77±0.62} \\ 
1,000 & 19.71±0.53 & 25.76±0.34 & 20.91±0.43 & 14.85±0.35 & 13.47±0.32 & 16.30±0.60 & 16.52±0.12 & 15.10±0.65 & 19.47±0.45 & \textbf{26.88±0.00} \\ 
3,000 & 39.06±0.21 & 39.98±0.23 & 38.51±0.15 & 32.40±0.40 & 28.80±0.45 & 32.52±0.63 & 32.97±0.34 & 31.60±0.60 & 37.00±0.65 & \textbf{43.40±0.29} \\ 
5,000 & 48.33±0.27 & 45.08±0.16 & 49.47±0.26 & 45.04±0.47 & 40.10±0.31 & 42.05±0.29 & 41.65±0.22 & 42.41±0.52 & 47.24±0.50 & \textbf{50.76±0.08} \\ 
7,000 & \textbf{55.12±0.09} & 47.34±0.44 & \textbf{55.25±0.34} & 53.48±0.38 & 49.48±0.19 & 51.30±0.68 & 49.74±0.21 & 49.67±0.46 & \textbf{54.78±0.29} & \textbf{55.17±0.34} \\ 
11,000 & 63.46±0.15 & 52.79±0.17 & \textbf{64.68±0.15} & 63.65±0.32 & 60.13±0.11 & 60.10±0.50 & 60.89±0.34 & 60.58±0.25 & 63.90±0.22 & 63.51±0.31 \\ 
15,000 & 69.42±0.20 & 55.04±0.06 & \textbf{70.27±0.25} & \textbf{70.08±0.21} & 67.38±0.17 & 67.49±0.19 & 66.34±0.16 & 66.65±0.23 & 69.36±0.24 & \textbf{70.19±0.45} \\ 
19,000 & 72.54±0.21 & 62.66±0.19 & \textbf{73.70±0.29} & \textbf{73.63±0.21} & 71.87±0.14 & 72.78±0.41 & 71.55±0.27 & 71.85±0.17 & 73.52±0.23 & \textbf{74.29±0.50} \\ 
23,000 & 75.99±0.16 & 70.06±0.27 & 76.65±0.11 & \textbf{77.17±0.06} & 75.30±0.24 & 76.37±0.29 & 75.73±0.25 & 75.64±0.27 & \textbf{77.12±0.30} & 76.68±0.20 \\ 
\bottomrule
\end{tabularx}
\begin{center} {ImageNet-50 dataset} \end{center}
% \end{table}
% ImageNet-100
% \begin{table}[H]
% \caption{Model accuracy for different sizes of labeled set $\sL$ and AL strategies} % on ImageNet-100 dataset}
% \centering
% \tiny
\setlength{\tabcolsep}{1pt} % Adjust horizontal space between cells
\begin{tabularx}{0.995\textwidth}{l*{11}{>{\centering\arraybackslash}X} r}
\toprule
$|\sL|$ & Random & Prob Cover & BADGE & BALD & Coreset & Uncertainty & Entropy & DBAL & Margin & DCoM \\ 
\midrule
200 & 4.58±0.13 & \textbf{6.72±0.07} & 4.67±0.16 & 1.99±0.22 & 4.58±0.13 & 2.01±0.06 & 1.99±0.10 & 1.83±0.06 & 1.84±0.12 & \textbf{6.87±0.18} \\ 
400 & 6.61±0.17 & 9.30±0.23 & 6.42±0.23 & 4.39±0.39 & 5.53±0.16 & 3.50±0.32 & 4.17±0.33 & 3.65±0.25 & 4.71±0.10 & \textbf{10.59±0.26} \\ 
600 & 8.11±0.42 & 12.71±0.42 & 8.01±0.20 & 6.42±0.24 & 6.45±0.37 & 4.93±0.29 & 5.51±0.33 & 4.96±0.22 & 6.46±0.12 & \textbf{13.61±0.15} \\ 
1,000 & 11.43±0.17 & \textbf{17.71±0.27} & 11.56±0.38 & 9.42±0.19 & 9.04±0.66 & 7.97±0.82 & 8.39±0.18 & 7.45±0.18 & 10.43±0.11 & \textbf{17.53±0.23} \\ 
3,000 & 27.59±0.76 & 34.01±0.15 & 27.58±0.72 & 24.96±0.23 & 19.41±0.32 & 23.10±0.60 & 20.97±0.33 & 20.88±0.52 & 26.51±0.79 & \textbf{34.85±0.23} \\ 
5,000 & 39.25±0.26 & 41.91±0.20 & 39.13±0.06 & 36.50±0.25 & 29.35±0.39 & 33.78±0.46 & 31.89±0.24 & 31.25±0.21 & 38.37±0.10 & \textbf{44.09±0.83} \\ 
7,000 & 46.88±0.25 & 48.41±0.42 & 46.42±0.41 & 43.49±0.42 & 38.45±0.12 & 42.09±0.54 & 41.19±0.17 & 40.47±0.83 & 45.69±0.26 & \textbf{49.22±0.29} \\ 
11,000 & 55.31±0.26 & 54.01±0.15 & 56.61±0.11 & 55.29±0.32 & 50.63±0.33 & 53.25±0.31 & 52.52±0.26 & 53.09±0.29 & 56.86±0.10 & \textbf{58.45±0.15} \\ 
15,000 & 61.69±0.09 & 57.04±0.43 & 63.05±0.42 & 62.19±0.21 & 58.71±0.05 & 60.07±0.07 & 59.76±0.22 & 59.64±0.30 & 62.71±0.27 & \textbf{64.57±0.22} \\ 
19,000 & 66.39±0.26 & 59.69±0.16 & 66.82±0.33 & 66.19±0.19 & 64.77±0.32 & 65.59±0.27 & 64.65±0.23 & 64.93±0.23 & 66.89±0.28 & \textbf{67.60±0.08} \\ 
23,000 & 68.92±0.15 & 61.99±0.41 & \textbf{70.54±0.21} & 69.62±0.32 & 68.49±0.06 & 69.03±0.20 & 68.47±0.50 & 67.65±0.14 & \textbf{70.81±0.35} & \textbf{70.21±0.32} \\ 
\bottomrule
\end{tabularx}
%\vspace{-.5cm}
\begin{center} {ImageNet-100 dataset} \end{center}
% \end{table}
% % ImageNet-200
% \begin{table}[H]
% \caption{Model accuracy for different sizes of labeled set $\sL$ and AL strategies on ImageNet-200 dataset}
% \centering
% \tiny
\setlength{\tabcolsep}{1pt} % Adjust horizontal space between cells
\caption{Model accuracy for different sizes of labeled set $\sL$ and AL strategies.}
\label{tab:res0}
\end{table*}

% STL-10
\begin{table*}[t!]
\centering
\scriptsize
\begin{tabularx}{0.995\textwidth}{l*{11}{>{\centering\arraybackslash}X} r}
\toprule
$|\sL|$ & Random & Prob Cover & BALD & Coreset & Uncertainty & Entropy & DBAL & Margin & DCoM \\ 
\midrule
400 & 3.10±0.07 & \textbf{5.28±0.08} & 1.18±0.21 & 3.10±0.07 & 1.19±0.07 & 1.17±0.13 & 1.29±0.06 & 2.22±0.50 & \textbf{5.20±0.02} \\ 
800 & 4.61±0.10 & 7.56±0.15 & 2.90±0.29 & 3.68±0.12 & 2.06±0.24 & 2.20±0.19 & 2.15±0.20 & 3.85±0.34 & \textbf{8.39±0.07} \\ 
1,200 & 5.88±0.25 & 10.10±0.08 & 4.32±0.33 & 4.84±0.14 & 3.42±0.19 & 3.71±0.26 & 3.42±0.16 & 5.54±0.23 & \textbf{11.56±0.07} \\ 
2,000 & 9.51±0.33 & 15.25±0.26 & 6.73±0.26 & 6.78±0.15 & 5.72±0.41 & 6.15±0.04 & 5.72±0.15 & 9.17±0.34 & \textbf{16.27±0.08} \\ 
6,000 & 25.67±0.16 & \textbf{32.89±0.12} & 19.85±0.26 & 17.98±0.10 & 20.12±0.11 & 19.89±0.35 & 18.72±0.08 & 26.03±0.08 & \textbf{32.76±0.13} \\ 
10,000 & 37.49±0.26 & \textbf{42.00±0.16} & 31.36±0.46 & 28.31±0.10 & 32.13±0.16 & 31.57±0.28 & 31.00±0.59 & 37.33±0.08 & \textbf{41.97±0.44} \\ 
15,000 & 46.22±0.30 & 48.89±0.09 & 41.16±0.49 & 38.53±0.18 & 42.53±0.29 & 41.88±0.31 & 41.42±0.13 & 46.37±0.17 & \textbf{49.47±0.27} \\ 
20,000 & 52.20±0.19 & 51.70±0.32 & 48.79±0.58 & 46.59±0.15 & 49.50±0.25 & 48.07±0.36 & 47.89±0.13 & 53.01±0.19 & \textbf{54.25±0.26} \\ 
40,000 & 64.57±0.17 & 57.60±0.05 & 64.02±0.37 & 61.78±0.30 & 63.94±0.27 & 63.14±0.37 & 63.21±0.13 & 65.19±0.24 & \textbf{65.69±0.17} \\ 
\bottomrule
\end{tabularx}
%\vspace{-.5cm}
\begin{center} {ImageNet-200 dataset} \end{center}
\setlength{\tabcolsep}{1pt} % Adjust horizontal space between cells
\begin{tabularx}{0.995\textwidth}{l*{11}{>{\centering\arraybackslash}X} r}
\toprule
$|\sL|$ & Random & Prob Cover & BADGE & BALD & Coreset & Uncertainty & Entropy & DBAL & Margin & DCoM \\ 
\midrule
10 & \textbf{15.73±0.90} & \textbf{17.26±0.78} & 14.92±0.98 & 12.22±0.91 & \textbf{15.58±1.03} & 11.88±0.83 & 11.67±0.48 & 11.79±1.06 & 12.81±0.65 & \textbf{16.72±0.94} \\ 
20 & 16.68±0.77 & \textbf{22.81±0.28} & 18.24±0.46 & 15.28±0.98 & 18.25±0.97 & 16.05±0.94 & 18.48±0.67 & 18.19±0.63 & 15.65±1.35 & \textbf{23.47±0.75} \\ 
30 & 19.78±0.79 & 24.91±0.49 & 19.80±0.38 & 18.37±1.06 & 19.41±0.92 & 17.75±1.05 & 19.33±0.74 & 20.35±0.90 & 18.49±1.16 & \textbf{26.72±0.79} \\ 
40 & 21.49±0.48 & 26.56±0.48 & 21.57±0.48 & 21.25±0.95 & 20.45±1.00 & 20.86±0.77 & 21.45±0.74 & 21.70±0.88 & 20.97±0.71 & \textbf{28.58±0.69} \\ 
50 & 22.19±0.33 & \textbf{28.30±0.42} & 24.40±0.45 & 22.13±0.56 & 21.76±1.07 & 22.11±0.81 & 22.75±0.60 & 22.76±0.91 & 23.57±0.59 & \textbf{29.20±0.87} \\ 
200 & 36.04±0.59 & 40.91±0.33 & 35.50±0.57 & 32.89±0.50 & 33.22±0.55 & 33.44±0.76 & 34.38±0.58 & 34.95±0.46 & 36.05±0.53 & \textbf{41.74±0.49} \\ 
400 & 44.07±0.60 & \textbf{49.52±0.23} & 44.18±0.54 & 43.33±0.45 & 41.83±0.39 & 42.39±0.64 & 42.51±0.46 & 43.67±0.47 & 44.35±0.52 & 47.02±0.73 \\ 
600 & 50.02±0.44 & \textbf{52.65±0.35} & 50.10±0.44 & 49.56±0.32 & 48.91±0.32 & 49.46±0.48 & 48.85±0.38 & 49.23±0.49 & 49.27±0.59 & \textbf{51.54±0.86} \\ 
800 & 53.69±0.36 & \textbf{54.45±0.56} & \textbf{54.57±0.63} & \textbf{54.20±0.51} & \textbf{54.42±0.21} & 53.48±0.38 & 53.70±0.38 & 53.98±0.31 & \textbf{54.05±0.56} & \textbf{55.05±0.71} \\ 
1,400 & 62.00±0.26 & 61.54±0.30 & \textbf{63.38±0.25} & \textbf{62.82±0.51} & \textbf{62.72±0.42} & 62.42±0.27 & 62.24±0.23 & \textbf{63.04±0.23} & \textbf{63.11±0.25} & \textbf{63.04±0.57} \\ 
2,600 & 70.18±0.27 & 70.50±0.08 & \textbf{72.76±0.19} & \textbf{72.72±0.26} & 71.41±0.13 & 72.53±0.23 & 72.35±0.18 & \textbf{72.97±0.18} & 72.12±0.40 & \textbf{72.93±0.33} \\ 
3,800 & 75.51±0.21 & 72.97±0.17 & \textbf{76.75±0.15} & 76.69±0.11 & 75.75±0.09 & \textbf{76.80±0.21} & \textbf{77.05±0.15} & \textbf{76.85±0.09} & \textbf{76.91±0.10} & \textbf{76.92±0.18} \\ 
\bottomrule
\end{tabularx}
\begin{center} {STL-10 dataset} \end{center}
% \end{table}
% % SVHN
% \begin{table}[H]
% \caption{Model accuracy for different sizes of labeled set $\sL$ and AL strategies} % on SVHN dataset}
% \label{tab:res2}
% \centering
% \tiny
\setlength{\tabcolsep}{1pt} % Adjust horizontal space between cells
\begin{tabularx}{0.995\textwidth}{l*{11}{>{\centering\arraybackslash}X} r}
\toprule
$|\sL|$ & Random & Prob Cover & BADGE & BALD & Coreset & Uncertainty & Entropy & DBAL & Margin & LDM-s & DCoM \\ 
\midrule
10 & 11.36±0.50 & 11.80±0.08 & \textbf{13.26±0.79} & \textbf{13.84±1.64} & 11.36±0.50 & \textbf{11.41±0.85} & \textbf{11.06±1.38} & \textbf{11.66±0.56} & 10.80±0.66 & - & \textbf{12.62±0.60} \\ 
20 & 11.05±0.41 & 10.96±0.19 & 12.19±0.67 & 12.74±1.26 & 10.99±0.59 & 11.81±0.84 & 10.86±0.58 & \textbf{14.94±0.36} & 12.50±0.95 & - & 13.30±0.26 \\ 
30 & 11.16±0.36 & \textbf{12.19±0.52} & \textbf{12.12±0.51} & \textbf{14.22±2.01} & 11.10±0.29 & \textbf{12.66±1.28} & \textbf{11.86±0.65} & \textbf{13.50±0.46} & \textbf{11.56±0.79} & - & \textbf{13.91±0.32} \\ 
40 & 11.85±0.46 & 12.11±0.20 & 12.71±0.70 & \textbf{13.98±2.06} & 11.57±0.57 & \textbf{14.16±1.63} & 12.28±0.47 & 13.10±0.41 & 11.76±0.75 & - & \textbf{14.32±0.14} \\ 
50 & 12.55±0.49 & 12.86±0.36 & 12.47±0.89 & \textbf{13.18±2.00} & 11.76±0.38 & \textbf{14.12±1.00} & 12.72±0.54 & 12.88±0.54 & 12.04±0.68 & - & \textbf{15.09±0.25} \\ 
200 & 22.32±0.84 & \textbf{24.34±0.44} & 20.18±1.09 & 19.00±1.66 & 17.42±1.03 & 21.99±1.11 & 20.38±1.08 & 20.98±0.99 & 22.03±0.87 & - & \textbf{24.58±1.10} \\ 
400 & 36.24±1.24 & 40.42±0.67 & 35.13±0.40 & 28.74±2.78 & 28.88±1.47 & 29.58±0.55 & 33.01±1.36 & 33.06±1.82 & 34.01±0.18 & - & \textbf{46.90±1.56} \\ 
600 & 47.29±0.98 & 52.67±0.75 & 46.46±0.87 & 37.26±1.56 & 40.24±1.24 & 42.65±0.68 & 42.89±1.38 & 43.99±0.46 & 46.69±0.50 & - & \textbf{58.34±1.25} \\ 
800 & 56.32±0.50 & 62.96±0.51 & 56.48±1.33 & 47.47±2.46 & 49.74±1.42 & 50.62±2.03 & 53.17±1.64 & 46.45±1.14 & 55.27±1.05 & - & \textbf{65.90±0.97} \\ 
1,400 & 74.00±0.38 & 73.16±0.21 & 75.53±0.77 & 65.33±3.30 & 68.44±0.69 & 72.13±0.98 & 72.81±1.09 & 73.14±0.69 & 74.53±0.74 & \textbf{79.02±0.01} & 78.36±0.36 \\ 
2,600 & 83.24±0.48 & 81.94±0.08 & \textbf{86.56±0.24} & 83.33±0.96 & 83.83±0.39 & 85.80±0.31 & 84.72±0.54 & 84.83±0.59 & \textbf{86.33±0.13} & 83.76±0.00 & \textbf{86.31±0.26} \\ 
3,800 & 87.18±0.14 & 85.34±0.11 & \textbf{89.61±0.25} & 88.91±0.54 & 88.86±0.14 & \textbf{89.64±0.09} & 88.79±0.14 & 89.36±0.15 & 89.34±0.18 & 86.08±0.00 & 89.28±0.13 \\ 
5,000 & 88.79±0.15 & 87.58±0.20 & \textbf{91.72±0.15} & 90.93±0.14 & 91.24±0.17 & \textbf{91.55±0.04} & 91.13±0.11 & 91.28±0.12 & \textbf{91.75±0.17} & 87.58±0.00 & \textbf{91.50±0.08} \\ 
6,200 & 90.16±0.04 & 88.55±0.10 & \textbf{92.64±0.09} & 92.46±0.04 & 92.35±0.03 & \textbf{92.63±0.08} & \textbf{92.56±0.16} & \textbf{92.65±0.10} & \textbf{92.55±0.12} & 88.64±0.00 & \textbf{92.54±0.12} \\ 
7,400 & 90.99±0.11 & 89.71±0.08 & \textbf{93.56±0.05} & 93.26±0.10 & 93.02±0.14 & 93.33±0.05 & 93.38±0.07 & 93.38±0.11 & \textbf{93.43±0.13} & 89.56±0.00 & \textbf{93.47±0.09} \\ 
\bottomrule
\end{tabularx}
\begin{center} {SVHN dataset} \end{center}

\caption{Model accuracy for different sizes of labeled set $\sL$ and AL strategies.}
\label{tab:res1}
%\vspace{5cm}
\end{table*}

\subsection{Empirical evaluation of $\boldsymbol{\cS}$}
\label{app:comp_score_illust}

Fig.~\ref{app:fig:comp_score_illust} tracks the behavior of the competence score $\cS$ across 3 datasets. In all cases, the comparison of Figs.\ref{app:fig:comp_score_illust:a},\ref{app:fig:comp_score_illust:b} reveals that $\cS$ increases monotonically with competence as measured by accuracy. Interestingly, while the critical point where $\cS$ changes from favoring $\mS$ (small values) to favoring $\mF$ (large values) corresponds to an intermediate level of accuracy, this transition is delayed for the more difficult datasets. 

\begin{figure}[htb]
\center
\begin{subfigure}{.49\columnwidth}
\includegraphics[width=\linewidth]{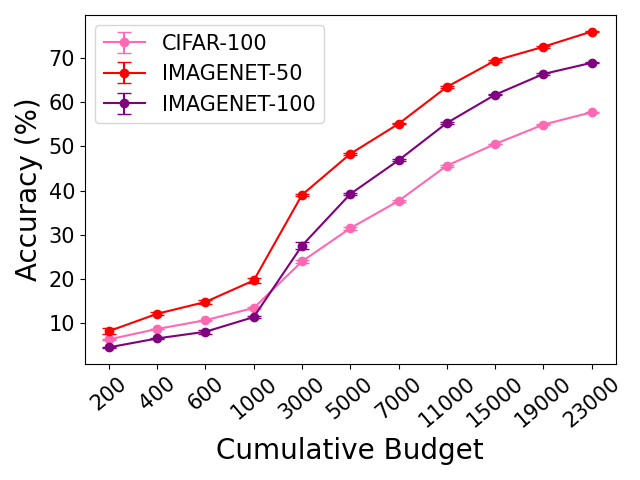}
\caption{Accuracy with random sampling.}
\label{app:fig:comp_score_illust:a}
\end{subfigure}
\begin{subfigure}{.49\columnwidth}
\includegraphics[width=\linewidth]{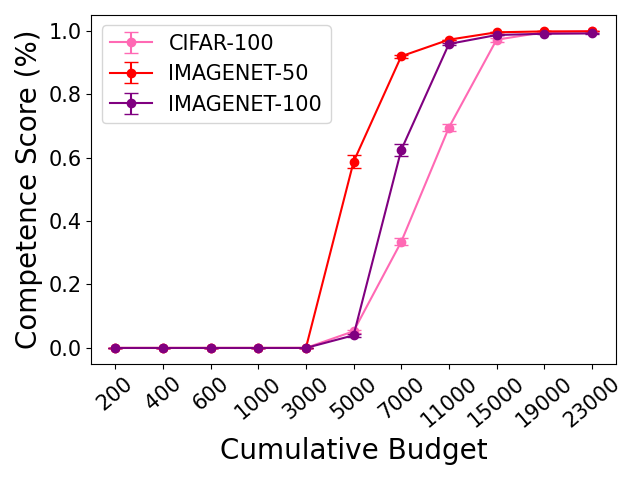}
\caption{$\cS$ during \MethodName\ training.}
\label{app:fig:comp_score_illust:b}
\end{subfigure}
\caption{(a) Accuracy (mean and standard error) as a function of budget $b$, learning to classify CIFAR-100, ImageNet-50, and ImageNet-100 using a random sample of $b$ points. (b) The competence score $\cS$ as a function of budget. }

\label{app:fig:comp_score_illust}
\end{figure}

\subsection{Different embedding space}
\label{app:different_emb_space}
In Section~\ref{sec:ablation}, we present an ablation study where we repeat the basic fully-supervised experiments while varying the embedding employed by \MethodName. In Fig.~\ref{app:fig:BYOL_ssl_results} you can see the same ablation using BYOL \citep{grill2020bootstrap} representation.

\begin{figure}[htb]
\centering
\begin{subfigure}{.49\columnwidth}
  \centering
  \includegraphics[width=\linewidth]{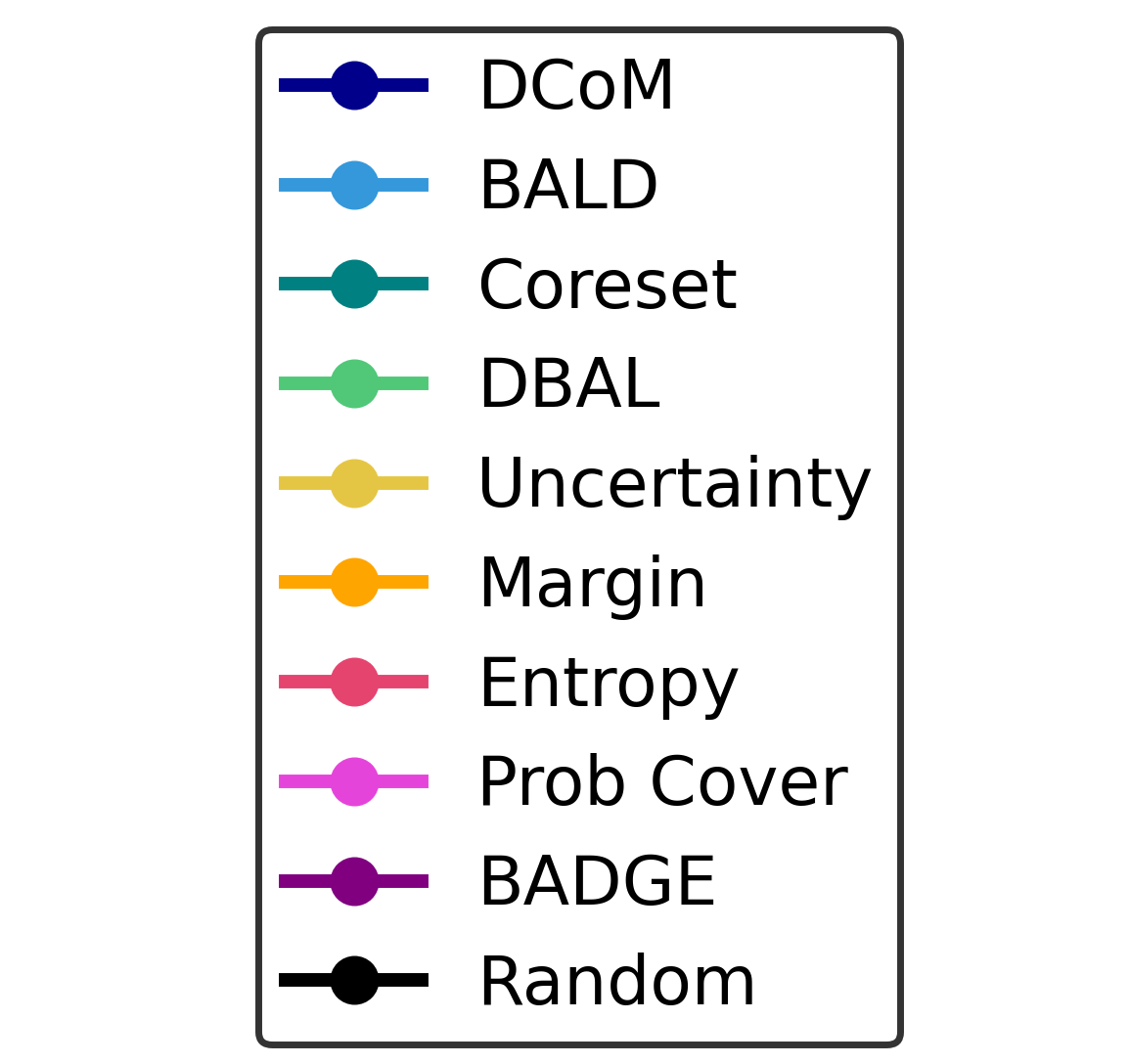}
\end{subfigure}
\begin{subfigure}{.49\columnwidth}
  \centering
  \includegraphics[width=\linewidth]{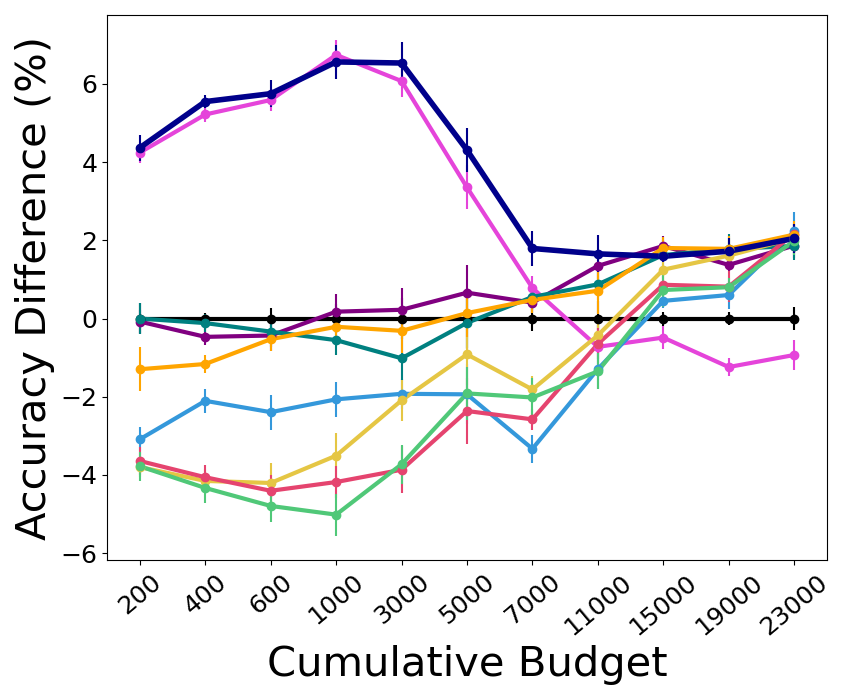}
\end{subfigure}
\caption{Performance over CIFAR-100, using BYOL feature space over \MethodName\ and \emph{ProbCover}, see details in the caption of Fig.~\ref{fig:main_results}. Clearly, \MethodName\ consistently achieves the best results.}
\label{app:fig:BYOL_ssl_results}
\end{figure}

%\clearpage
\subsection{Comparison per dataset}
\label{app:comp_per_dataset}
Tables~\ref{tab:res0}-\ref{tab:res1} present the empirical accuracy values of several datasets using different active learning algorithms. In each active step, a new learner is trained from scratch over all available labeled data, and the results are presented as the [mean $\pm$ STE] over $5-10$ repetitions (3 for ImageNet subsets). As mentioned earlier, the results for LDM-s algorithm were acquired from the study by \cite{cho2023querying}. Their lack of code provision suggests discrepancies in the running setup. Nonetheless, the disparity between their method and random selection in their setup is smaller than that observed in ours, implying the superiority of our method. The table includes the results for all datasets, with the name of each dataset below it.

%\section*{Acknowledgements (optional, unnumbered)}
%\inbal{todo}

\end{document}